\documentclass{article}

\usepackage[preprint]{neurips_2021}

\usepackage[utf8]{inputenc}

\usepackage{url}
\usepackage[T1]{fontenc}    
\usepackage[colorlinks]{hyperref}       
\usepackage{url}            
\usepackage{booktabs}       
\usepackage{amsfonts}       
\usepackage{nicefrac}       
\usepackage{microtype}      
\usepackage{amssymb}
\usepackage{amsmath}
\usepackage{amsthm}
\usepackage{natbib}
\usepackage{color}
\usepackage{graphicx}
\usepackage{colordvi}
\usepackage{mathtools}
\usepackage{subcaption}
\usepackage{bm}
\usepackage{bbold}
\usepackage{wrapfig}
\usepackage[capitalise]{cleveref}

\usepackage{natbib}

\usepackage{xcolor}
\definecolor{niceblue}{rgb}{0.10, 0.14, 0.76} 

\AtBeginDocument{%
\hypersetup{
    citecolor=niceblue,
    linkcolor=red,   
    urlcolor=niceblue}
}

\title{Regularization in ResNet with Stochastic Depth}

\author{%
  Soufiane Hayou\thanks{Equal contribution.
  Correspondence to: \texttt{<soufiane.hayou@yahoo.fr; fadhel.ayed@huawei.com>}} \\
  Department of Statistics\\
  University of Oxford\\
  United Kingdom \\
  \And
  Fadhel Ayed$^*$ \\
  Huawei Technologies \\
  France \\
}

\newtheorem{prop}{Proposition}
\newtheorem{thm}{Theorem}
\newtheorem{lemma}{Lemma}
\newtheorem{corollary}{Corollary}

\newtheorem{assumption}{Assumption}

\newtheorem{manlemmin}{Lemma}
\newenvironment{manuallemma}[1]{%
  \renewcommand\themanlemmin{#1}%
  \manlemmin
}{\endmanlemmin}

\newtheorem{manthmin}{Theorem}
\newenvironment{manualthm}[1]{%
  \renewcommand\themanthmin{#1}%
  \manthmin
}{\endmanthmin}

\newtheorem{manpropmin}{Proposition}
\newenvironment{manualprop}[1]{%
  \renewcommand\themanpropmin{#1}%
  \manthmin
}{\endmanpropmin}

\newcommand{\loss}{\mathcal{L}}
\newcommand{\reals}{\mathbb{R}}
\newcommand{\normalD}{\mathcal{N}}
\newcommand{\E}{\mathbb{E}}

\newcommand{\Lp}{L_{\bm{p}}}
\newcommand{\Ldelta}{L_{\bm{\delta}}}
\newcommand{\mask}{\bm{\delta}}
\newcommand{\probs}{\bm{p}}
\newcommand{\Var}{\textup{Var}}
\newcommand{\ones}{\bm{1}}
\newcommand{\data}{\mathcal{D}}
\newcommand{\X}{\mathcal{X}}
\newcommand{\T}{\mathcal{T}}
\newcommand{\weights}{\bm{W}}

\newcommand{\neuron}{y^i_{\alpha L}}
\newcommand{\SD}{$\mathcal{SD}$}

\begin{document}

\maketitle

\begin{abstract}
Regularization plays a major role in modern deep learning. From classic techniques such as $L_1, L_2$ penalties to other noise-based methods such as Dropout, regularization often yields better generalization properties by avoiding overfitting. Recently, Stochastic Depth (\SD) has emerged as an alternative regularization technique for residual neural networks (ResNets) and has proven to boost the performance of ResNet on many tasks \citep{huang2016stochasticdepth}. Despite the recent success of \SD, little is known about this technique from a theoretical perspective. This paper provides a hybrid analysis combining perturbation analysis and signal propagation to shed light on different regularization effects of \SD. Our analysis allows us to derive principled guidelines for choosing the survival rates used for training with \SD.
\end{abstract}

\section{Introduction}

Stochastic Depth (\SD) is a well-established regularization method that was first introduced by \cite{huang2016stochasticdepth}. It is similar in principle to Dropout \citep{hinton2012improving, srivastava2014dropout} and DropConnect \citep{wan2013regularization}. It belongs to the family of noise-based regularization techniques, which includes other methods such as noise injection in data \citep{webb1994noise, bishop1995noise} and noise injection throughout the network \citep{camuto2020noiseinjection}. While Dropout, resp.  DropConnect consists of removing some neurons, resp. weights, at each iteration, \SD~randomly drops \emph{full layers}, and only updates the weights of the resulting subnetwork at each training iteration. As a result of this mechanism, \SD~can be exclusively used with residual neural networks (ResNets).

There exists a stream of papers in the literature on the regularization effect of Dropout for linear models \citep{wager2013dropout, mianjy2019dropout, helmbold2015inductive, cavazza2017dropout}. Recent work by \cite{wei2020dropout} extended this analysis to deep neural networks using second-order perturbation analysis. It disentangled the explicit regularization of Dropout on the loss function and the implicit regularization on the gradient. Similarly, \cite{camuto2020noiseinjection} studied the explicit regularization effect induced by adding Gaussian Noise to the activations and empirically illustrated the benefits of this regularization scheme. However, to the best of our knowledge, no analytical study of \SD~exists in the literature. This paper aims to fill this gap by studying the regularization effect of \SD~from an analytical point of view; this allows us to derive principled guidelines on the choice of the survival probabilities for network layers. Concretely, our contributions are four-fold:
\begin{itemize}
    \item We show that \SD~acts as an explicit regularizer on the loss function by penalizing a notion of \textit{information discrepancy} between keeping and removing certain layers.
    \item We prove that the \emph{uniform mode}, defined as the choice of constant survival probabilities, is related to maximum regularization using \SD.
    \item We study the large depth behaviour of \SD~ and show that in this limit, \SD~mimics \emph{Gaussian Noise Injection} by implicitly adding data-adaptive Gaussian noise to the pre-activations.
    \item By defining the training budget $\bar{L}$ as the \emph{desired} average depth, we show the existence of two different regimes: \emph{small budget} and \emph{large budget} regimes. We introduce a new algorithm called \emph{SenseMode} to compute the survival rates under a fixed training budget and provide a series of experiments that validates our \emph{Budget hypothesis} introduced in \cref{sec:regimes}.
\end{itemize}

\section{Stochastic Depth Neural Networks}
Stochastic depth neural networks were first introduced by \cite{huang2016stochasticdepth}. They are standard residual neural networks with random depth. In practice, each block in the residual network is multiplied by a random Bernoulli variable $\delta_l$ ($l$ is the block's index) that is equal to $1$ with some survival probability $p_l$ and $0$ otherwise. The mask is re-sampled after each training iteration, making the gradient act solely on the subnetwork composed of blocks with $\delta_l = 1$.

We consider a slightly different version where we apply the binary mask to the pre-activations instead of the activations. We define a depth $L$ stochastic depth ResNet by
\begin{equation}\label{eq:standard_resnet}
\begin{aligned}
y_0(x; \bm{\delta}) &= \Psi_0(x, W_0), \\
y_l(x; \bm{\delta}) &= y_{l-1}(x; \bm{\delta}) + \delta_l \Psi_l(y_{l-1}(x; \bm{\delta}), W_l), \quad 1\leq l \leq L,\\
y_{out}(x; \bm{\delta}) &= \Psi_{out}(y_{L}(x; \bm{\delta}), W_{out}),
\end{aligned}
\end{equation}
where $W_l$ are the weights in the $l^{th}$ layer, $\Psi$ is a mapping that defines the nature of the layer, $y_l$ are the pre-activations, and $\bm{\delta}=(\delta_l)_{1 \leq l \leq L}$ is a vector of Bernoulli variables with survival parameters $\bm{p}=(p_{l})_{1\leq l \leq L}$.
$\mask$ is re-sampled at each iteration. For the sake of simplification, we consider constant width ResNet and we further denote by $N$ the width, i.e. for all $l\in[L-1]$, $y_l \in \reals^{N}$. The output function of the network is given by $s(y_{out})$ where $s$ is some convenient mapping for the learning task, e.g. the Softmax mapping for classification tasks. We denote by $o$ the dimension of the network output, i.e. $s(y_{out}) \in \reals^{o}$ which is also the dimension of $y_{out}$. 
For our theoretical analysis, we consider a Vanilla model with residual blocks composed of a Fully Connected linear layer
$$
\Psi_l(x, W) = W \phi (x),
$$
where $\phi(x)$ is the activation function. The weights are initialized with He init \citep{he_init}, e.g. for ReLU, $W^l_{ij} \sim \mathcal{N}(0,2/N)$. 

There are no principled guidelines on choosing the survival probabilities. However, the original paper by \cite{huang2016stochasticdepth} proposes two alternatives that appear to make empirical consensus: the \emph{uniform} and \emph{linear modes}, described by
\begin{center}
  \textbf{Uniform:}\ \ $p_l = p_L$, \quad 
  \textbf{Linear:}\ \ $p_l = 1 - \frac{l}{L}(1-p_L)$, 
\end{center}
where we conveniently parameterize both alternatives using $p_L$.
\paragraph{Training budget $\bar{L}$.}  We define the training budget $\bar{L}$ to be the desired average depth of the subnetworks with \SD. The user typically fixes this budget, e.g., a small budget can be necessary when training is conducted on small capacity devices.

\paragraph{Depth of the subnetwork.} Given the mode $\bm{p}=(p_l)_{1 \leq l \leq L}$, after each iteration, the subnetwork has a depth $L_{\bm{\delta}} = \sum_{l=1}^L \delta_l$ with an average $L_{\bm{p}}:= \E_{\bm{\delta}}[L_{\bm{\delta}}] = \sum_{l=1}^L p_l$. Given a budget $\bar{L}$, there is a infinite number of modes $\probs$ such that $\Lp = \bar{L}$. In the next lemma, we provide probabilistic bounds on $L_{\bm{\delta}}$ using standard concentration inequalities. We also show that with a fixed budget $\bar{L}$, the uniform mode is linked to maximal variability.
\begin{lemma}[Concentration of  $L_{\bm{\delta}}$]\label{Lemma:bounds_L}
For any $\beta \in (0,1)$, we have that with probability at least $1-\beta$,
\begin{equation}\label{eq:bennett}
|L_{\bm{\delta}} - L_{\bm{p}}| \leq v_{\bm{p}} \,  u^{-1}\left(\frac{\log(2/\beta)}{v_{\bm{p}}}\right),
\end{equation}
where $L_{\bm{p}} = \mathbb{E}[L_{\bm{\delta}}] = \sum_{l=1}^L p_l$, $v_{\bm{p}} = \textup{Var}[L_{\bm{\delta}}] = \sum_{l=1}^L p_l(1- p_l)$, and $u(t) = (1+t) \log(1+t) - t$.

Moreover, for a given average depth $\Lp=\bar{L}$, the upperbound in \cref{eq:bennett} is maximal for the uniform choice of survival probabilities $\bm{p} = \left(\frac{\bar{L}}{L}, ..., \frac{\bar{L}}{L}\right)$.
\end{lemma}

\begin{wrapfigure}{r}{0.36\textwidth}
	\centering
	\begin{subfigure}[t]{.36\textwidth}
		\centering
  \includegraphics[width=1.\textwidth]{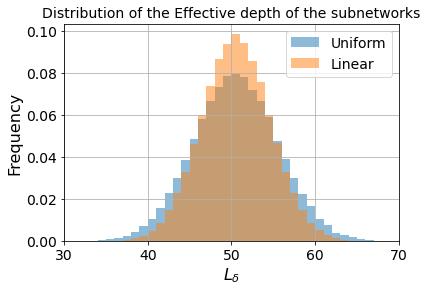}
  \caption{\small{Distributions of $\Ldelta$ for a Resnet100 with average survival rate $\bar L/ L = 0.5$ for the uniform and linear modes.}}
  \label{fig:eff_depth_distrib}
	\end{subfigure}
	\quad
	\begin{subfigure}[t]{.36\textwidth}
		    \centering
    \includegraphics[width=1.\textwidth]{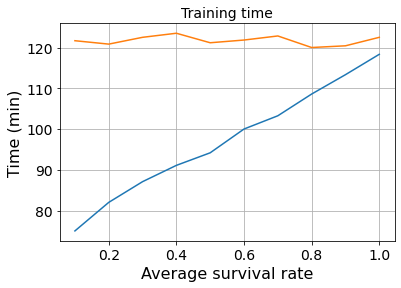}
    \caption{\small{Training time of Dropout and \SD~on CIFAR10 with ResNet56 for 100 epochs.}}
    \label{fig:training_time}
	\end{subfigure}
    \caption{}
	\vspace{-1.2cm}
\end{wrapfigure}

\cref{Lemma:bounds_L} shows that with high probability, the depth of the subnetwork that we obtain with \SD~ is within an $\ell_1$ error of $v_{\bm{p}} \,  u^{-1}\left(\frac{\log(2/\beta)}{v_{\bm{p}}}\right)$ from the average depth $\Lp$. Given a fixed budget $\bar{L}$, this segment is maximized for the uniform mode $\bm{p}=(\bar{L}/L, \dots, \bar{L}/L)$. \cref{fig:eff_depth_distrib} highlights this result. This was expected since the variance of the depth $\Ldelta$ is also maximized by the uniform mode. The uniform mode corresponds to maximum entropy of the random depth, which would intuitively results in maximum regularization. We depict this behaviour in more details in \cref{Sec:regularization}.
\paragraph{\SD~vs Dropout.} From a computational point of view, \SD~ has the advantage of reducing the effective depth during training. Depending on the chosen budget, the subnetworks might be significantly shallower than the entire network (\cref{Lemma:bounds_L}). This depth compression can be effectively leveraged for computational training time gain (\cref{fig:training_time}). It is not the case with Dropout. Indeed, assuming that the choice of dropout probabilities is such that we keep the same number of parameters on average compared to \SD, we still have to multiply matrices $L$ times during the forward/backward propagation. It is not straightforward to leverage the sparsity obtained by Dropout for computational gain. In practice, Dropout requires an additional step of sampling the mask for every neuron, resulting in longer training times than without Dropout (\cref{fig:training_time}). However, there is a trade-off between how small the budget is and the performance of the trained model with \SD~ (\cref{sec:experiments}).


\section{Effect of Stochastic Depth at initialization}\label{Sec:large_width}
Empirical evidence strongly suggests that Stochastic Depth allows training deeper models \citep{huang2016stochasticdepth}. Intuitively, at each iteration, \SD~updates only the parameters of a subnetwork with average depth $\Lp = \sum_l p_l < L$, which could potentially alleviate any exploding/vanishing gradient issue. This phenomenon is often faced when training ultra deep neural networks. To formalize this intuition, we consider the model's asymptotic properties at initialization in the infinite-width limit $N \rightarrow +\infty$. This regime has been the focus of several theoretical studies \citep{neal, poole, samuel, yang_tensor3_2020, xiaocnnmeanfield, hayou, hayou_ntk, hayou_stable_resnet} since it allows to derive analytically the distributions of different quantities of untrained neural networks. Specifically, randomly initialized ResNets, as well as other commonly-used architectures such as Fully connected Feedforward networks, convolutional networks and LSTMs, are equivalent to Gaussian Processes in the infinite-width limit. An important ingredient in this theory is the \emph{Gradient Independence} assumption. Let us formally state this assumption first.
\begin{assumption}[Gradient Independence]\label{assumption:gradient_independence}
In the infinite width limit, we assume that the weights $\weights$ used for back-propagation are an iid version of the weights $\weights$ used for forward propagation. 
\end{assumption}
\cref{assumption:gradient_independence} is ubiquitous in the literature on the signal propagation in deep neural networks. It has been used to derive theoretical results on signal propagation in randomly initialized deep neural network \citep{samuel, poole, yang2017meanfield, hayou_pruning, hayou_stable_resnet} and is also a key tool in the derivation of the Neural Tangent Kernel \citep{jacot, arora2019exact, hayou_ntk}. Recently, it has been shown by \cite{yang_tensor3_2020} that \cref{assumption:gradient_independence} yields the exact computation for the gradient covariance in the infinite width limit. See Appendix \ref{app:discussion_assumption} for a detailed discussion about this assumption. Throughout the paper, we provide numerical results that substantiate the theoretical results that we derive using this assumption. We show that \cref{assumption:gradient_independence} yields an excellent match between theoretical results and numerical experiments.

Leveraging this assumption, \cite{yang2017meanfield, hayou_stable_resnet} proved that ResNet suffers from exploding gradient at initialization. We show in the next proposition that \SD~helps mitigate the exploding gradient behaviour at initialization in infinite width ResNets.

\begin{wraptable}{r}{4.5cm}
\vspace{-2.0cm}
\caption{\small{Gradient magnitude growth rate with Vanilla Resnet50 with width 512 and training budget $\bar L / L = 0.7$. Empirical vs. Theoretical value (between parenthesis) of $\tilde{q}_l(x, z)$ at initialization, for standard (no \SD), uniform and linear modes. The expectation is performed using 500 MC samples.}}
  \label{tab:growth_07}
  \centering
  \resizebox{4.5cm}{!}{%
  \begin{tabular}{cccc}
    \toprule
     & Standard & Uniform & Linear \\
     $\ell$ & & & \\
    \midrule
0 &	2.001 (2) &	1.705 (1.7)&	1.694 (1.691) \\
10 & 2.001 (2)&	1.708 (1.7)&	1.633 (1.629)\\
20 & 2.001 (2)&	1.707 (1.7)&	1.569 (1.573) \\
30 & 2.001 (2)&	1.716 (1.7)&	1.555 (1.516) \\
40 & 1.999 (2)&	1.739 (1.7)&	1.530 (1.459)\\
    \bottomrule
  \end{tabular}
  }
  \vspace{-0.2cm}
\end{wraptable} 
\begin{prop}\label{prop:exploding_gradient}
Let $\phi=\textrm{ReLU}$ and $\loss(x,z) = \ell(y_{out}(x; \bm{\delta}), z)$ for $(x, z) \in \reals^d \times \reals^o$, where $\ell(z,z')$ is some differentiable loss function. Let $\tilde{q}_l(x,z)= \E_{W,\bm{\delta}} \frac{\lVert \nabla_{y_l} \loss \rVert^2}{ \lVert \nabla_{y_L} \loss \rVert^2}$, where the numerator and denominator are respectively the norms of the gradients with respect to the inputs of the $l^{th}$ and $L^{th}$ layers . Then, in the infinite width limit, under \cref{assumption:gradient_independence}, for all $l \in [L]$ and $(x,z) \in \reals^d \times \reals^o$, we have 
\begin{itemize}
    \item With Stochastic Depth, $ \tilde{q}_l(x,z)  = \prod_{k=l+1}^L (1+p_k)$,
    \item Without Stochastic Depth (i.e. $\mask=\ones$), $ \tilde{q}_l(x,z)  = 2^{L-l} $.
\end{itemize}
\end{prop}

\cref{prop:exploding_gradient} indicates that with or without \SD, the gradient explodes exponentially at initialization as it backpropagates through the network. However, with \SD, the exponential growth is characterized by the mode $\probs$. Intuitively, if we choose $p_l\ll 1$ for some layer $l$, then the contribution of this layer in the exponential growth is negligible since $1+p_l \approx 1$. From a practical point of view, the choice of $p_l\ll 1$ means that the $l^{th}$ layer is hardly present in any subnetwork during training, thus making its contribution to the gradient negligible on average (w.r.t $\bm{\delta}$). For a ResNet with $L=50$ and uniform mode $\probs=(1/2, \dots, 1/2)$, \SD~reduces the gradient exploding by six orders of magnitude. \cref{fig:gradgrowth_standard_0.5} and \cref{fig:gradgrowth_standard_0.7} illustrates the exponential growth of the gradient for the uniform and linear modes, as compared to the growth of the gradient without \SD. We compare the empirical/theoretical growth rates of the magnitude of the gradient in \cref{tab:growth_07}; the results show a good match between our theoretical result (under \cref{assumption:gradient_independence}) and the empirical ones. Further analysis can be found in \cref{app:further_experimental_results}.
\vspace{-0.3cm}
\paragraph{Stable ResNet.}\cite{hayou_stable_resnet} have shown that introducing the scaling factor $1/\sqrt{L}$ in front of the residual blocks is sufficient to avoid the exploding gradient at initialization, as illustrated in Figure \ref{fig:gradgrowth_stable_0.7}. The hidden layers in Stable ResNet (with \SD) are given by,
\begin{equation}\label{eq:stable_resnet}
\begin{aligned}
y_l(x; \bm{\delta}) &= y_{l-1}(x; \bm{\delta}) + \frac{\delta_l}{\sqrt{L}} \Psi_l(y_{l-1}(x; \bm{\delta}), W_l), \quad 1\leq l \leq L.\\
\end{aligned}
\end{equation}
The intuition behind the choice of the scaling factor $1/\sqrt{L}$ comes for the variance of $y_l$. At initialization, with standard ResNet (\cref{eq:standard_resnet}), we have $\Var[y_l] = \Var[y_{l-1}] + \Theta(1)$, which implies that $\Var[y_l]=\Theta(l)$. With Stable ResNet (\cref{eq:stable_resnet}), this becomes $\Var[y_l] = \Var[y_{l-1}] + \Theta(1/L)$, resulting in $\Var[y_l]=\Theta(1)$ (See \cite{hayou_stable_resnet} for more details). In the rest of the paper, we restrict our analysis to Stable ResNet; this will help isolate the regularization effect of \SD~in the limit of large depth without any variance/gradient exploding issue.
\begin{figure}	
	\centering
	\begin{subfigure}[t]{.31\textwidth}
		\centering
		\includegraphics[width=\textwidth]{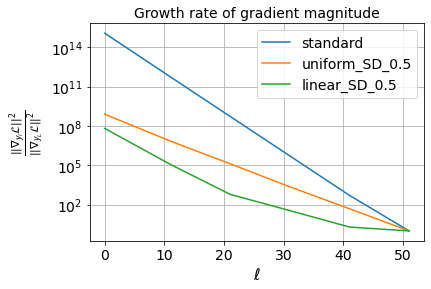}
		\caption{$\bar L / L = 0.5$, standard ResNet.}
		\label{fig:gradgrowth_standard_0.5}
	\end{subfigure}
	\begin{subfigure}[t]{.31\textwidth}
		\centering
		\includegraphics[width=\textwidth]{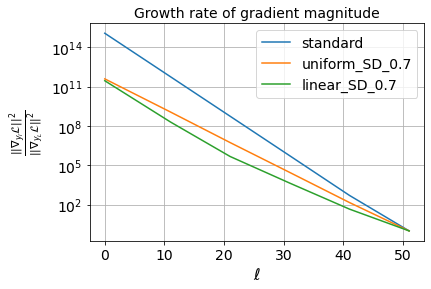}
		\caption{$\bar L / L = 0.7$, standard ResNet.}\label{fig:gradgrowth_standard_0.7}
	\end{subfigure}
	\begin{subfigure}[t]{.31\textwidth}
		\centering
        \includegraphics[width=\textwidth]{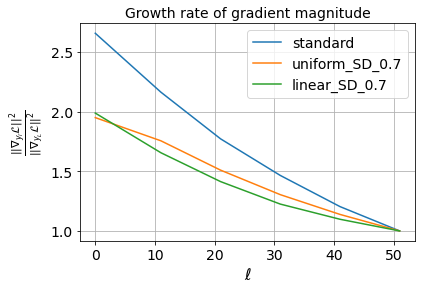}
        \caption{$\bar L / L = 0.7$, stable ResNet.}
        \label{fig:gradgrowth_stable_0.7}
	\end{subfigure}
	\caption{\small{Empirical illustration of Proposition 1 ((a) and (b)) and Stable Resnet (c). Comparison of the growth rate of the gradient magnitude $\tilde{q}_l(x, z)$ at initialization for Vanilla ResNet50 with width 512. The y-axis of figures (a) and (b) are in log scale. The y-axis of figure (c) is in linear scale. The expectation is computed using 500 Monte-Carlo (MC) samples.}}
	\vspace{-0.3cm}
\end{figure}

Nevertheless, the natural connection between \SD~and Dropout, coupled with the line of work on the regularization effect induced by the latter \citep{wager2013dropout, mianjy2019dropout, helmbold2015inductive, cavazza2017dropout, wei2020dropout}, would indicate that the benefits of \SD~are not limited to controlling the magnitude of the gradient. Using a second order Taylor expansion, \cite{wei2020dropout} have shown that Dropout induces an explicit regularization on the loss function. Intuitively, one should expect a similar effect with \SD. In the next section, we elucidate the \emph{explicit} regularization effect of \SD~on the loss function, and we shed light on another regularization effect of \SD~that occurs in the large depth limit.

\section{Regularization effect of Stochastic}\label{Sec:regularization}
\subsection{Explicit regularization on the loss function}
Consider a dataset $\data = \X \times \T$ consisting of $n$ (input, target) pairs $\{(x_i, t_i)\}_{1 \leq i \leq n}$ with $(x_i, t_i) \in \reals^d \times \reals^o$. Let $\ell:\reals^d \times \reals^o \to \reals$ be a smooth loss function, e.g. quadratic loss, crossentropy loss etc. Define the model loss for a single sample $(x,t) \in \data$ by
$$
\loss(\weights,x; \mask) =  \ell(y_{out}(x;\mask), t), \quad \loss(\weights,x) = \E_{\delta} \left[ \ell(y_{out}(x;\mask), t)\right],
$$
where $\weights=(W_l)_{0\leq l \leq L}$. The empirical loss given by
$
\loss(\weights) = \frac{1}{n} \sum_{i=1}^n \E_{\delta} \left[ \ell(y_{out}(x_i;\mask), t_i)\right].
$

To isolate the regularization effect of \SD~on the loss function, we use a second order approximation of the loss function around $\mask=\ones$, this allows us to marginalize out the mask $\mask$. The full derivation is provided in Appendix \ref{app:full_derivation_explicit_reg}. Let $z_l(x;\mask) = \Psi_l(W_l, y_{l-1}(x; \mask))$ be the activations. For some pair $(x,t) \in \data$, we obtain
\vspace{-0.4cm}
\begin{equation}\label{eq:approximation_loss}
\loss(\bm{W},x) \approx \bar{\loss}(\bm{W},x) + \frac{1}{2L} \sum_{l=1}^L p_l(1 - p_l) g_l(\weights, x),
\end{equation}
where $\bar{\loss}(\bm{W}, x) \approx \ell(y_{out}(x;\probs), t)$ (more precisely, $\bar{\loss}(\bm{W}, x)$ is the second order Taylor approximation of $\ell(y_{out}(x;\probs), t)$ around $\probs=1$\footnote{Note that we could obtain \cref{eq:approximation_loss} using the Taylor expansion around $\mask=\probs$. However, in this case, the hessian will depend on $\probs$, which complicates the analysis of the role of $\probs$ in the regularization term.}), and $g_l(\weights, x) =  z_l(x; \ones)^{T} \nabla^2_{y_l}[\ell\circ G_l](y_{l}(x;\ones)) z_l(x; \ones)$ with $G_l$ is the function defined by $y_{out}(x;\ones) = G_l(y_{l-1}(x;\ones) + \frac{1}{\sqrt{L}}z_l(x;\ones))$. 

The first term $\bar{\loss}(\weights, x)$ in \cref{eq:approximation_loss} is the loss function of the average network (i.e. replacing $\mask$ with its mean $\probs$). Thus, \cref{eq:approximation_loss} shows that training with \SD~entails training the average network with an explicit regularization term that implicitly depends on the weights $\weights$.
\vspace{-0.2cm}
\paragraph{\SD~enforces flatness.} The presence of the hessian in the penalization term provides a geometric interpretation of the regularization induced by \SD: it enforces a notion of flatness determined by the hessian of the loss function with respect to the hidden activations $z_l$. This flatness is naturally inherited by the weights, thus leading to flatter minima. Recent works by \citep{keskar2016large, jastrzebski2018relation, yao2018hessian} showed empirically that flat minima yield better generalization compared to minima with large second derivatives of the loss. \cite{wei2020dropout} have shown that a similar behaviour occurs in networks with Dropout.

Let $J_{l}(x) = \nabla_{y_l} G_l (y_{l}(x;\ones))$ be the Jacobian of the output layer with respect to the hiden layer $y_l$ with $\mask=\ones$, and $H_{\ell}(x) = \nabla^2_z \ell(z)_{|z=y_{out}(x;\ones)}$ the hessian of the loss function $\ell$. The hessian matrix inside the penalization terms $g_l(\weights, x)$ can be decomposed as in \citep{lecun2012efficient, sagun2017empirical}
$$
\nabla^2_{y_l}[\ell\circ G_l](y_{l}(x;\ones)) = J_{l}(x)^{T} H_{\ell}(x) J_{l}(x) + \Gamma_l(x),
$$
where $\Gamma$ depends on the hessian of the network output. $\Gamma$ is generally non-PSD, and therefore cannot be seen as a regularizer. Moreover, it has been shown empirically that the first term generally dominates and drives the regularization effect \citep{sagun2017empirical, wei2020dropout, camuto2020noiseinjection}. Thus, we restrict our analysis to the regularization effect induced by the first term, and we consider the new version of $g_l$ defined by
\begin{equation}\label{eq:penalization_term}
g_l(\bm{W}, x) =  \zeta_l(x, \weights)^{T} \, H_{\ell}(x) \, \zeta_l(x, \weights) =  \textup{Tr}\left(H_{\ell}(x) \, \zeta_l(x, \weights)\zeta_l(x, \weights)^{T}\right),
\end{equation}
where $\zeta_l(x, \weights)= J_{l}(x)z_l(x; \ones)$. The quality of this approximation is discussed in Appendix \ref{app:further_experimental_results}. 
\vspace{-0.2cm}
\paragraph{Information discrepancy.} The vector $\zeta_l$ represents a measure of the information discrepancy between keeping and removing the $l^{th}$ layer. Indeed, $\zeta_l$ measures the sensitivity of the model output to the $l^{th}$ layer,
$$
y_{out}(x;\ones) - y_{out}(x;\ones_{l}) \approx \nabla_{\delta_l} y_{out}(x;\mask)_{|\mask=\ones} = \zeta_l(x, \weights),
$$
where $\ones_{l}$ is the vector of $1's$ everywhere with $0$ in the $l^{th}$ coordinate.\\
With this in mind, the regularization term $g_l$ in \cref{eq:penalization_term} is most significant when the information discrepancy is well-aligned with the hessian of the loss function, i.e. \SD~penalizes mostly the layers with information discrepancy that violates the flatness, confirming our intuition above. 
\paragraph{Quadratic loss.} With the quadratic loss $\ell(z,z') = \|z-z'\|_2^2$, the hessian $H_{\ell}(x) = 2 I$ is isotropic, i.e. it does not favorite any direction over the others. Intuitively, we expect the penalization to be similar for all the layers. In this case,  we have 
$
g_l(\weights, x) = 2 \, \|\zeta_l(x)\|_2^2,
$
and the loss is given by 
\begin{equation}\label{eq:empirical_loss}
\loss(\weights) \approx \bar{\loss}(\weights) + \frac{1}{2 L} \sum_{l=1}^L p_l(1 - p_l) g_l(\weights),
\end{equation}
where $g_l(\weights) = \frac{2}{n} \sum_{i=1}^n \|\zeta_l(x_i, \weights)\|_2^2\,$~is the regularization term marginalized over inputs $\X$.

\cref{eq:empirical_loss} shows that the mode $\probs$ has a direct impact on the regularization term induced by \SD. The latter tends to penalize mostly the layers with survival probability $p_l$ close to $50\%$. The mode $\probs=(1/2,\dots,1/2)$ is therefore a universal maximizer of the regularization term, given fixed weights $\weights$. However, given a training budget $\bar{L}$, the mode  $\probs$ that maximizes the regularization term $\frac{1}{2L}\sum_{l=1}^L p_l(1-p_l) g_l(W)$ depends on the values of $g_l(\weights)$. We show this in the next lemma.

\begin{lemma}[Max regularization]\label{lemma:maximal_regularization} Consider the empirical loss $\loss$ given by \cref{eq:empirical_loss} for some fixed weights $\weights$ (e.g. $\weights$ could be the weights at any training step of SGD). Then, given a training budget $\bar{L}$, the regularization is maximal for 
$p^*_l = \min \big( 1, \max (0, \frac{1}{2} - C g_l(\weights)^{-1} ) \big),$
where $C$ is a normalizing constant, that has the same sign as $L-2 \bar{L}$. The global maximum is obtained for $p_l = 1/2$.
\end{lemma}
\begin{wrapfigure}{r}{0.35\textwidth}
  \centering
  \vspace{-0.6cm}
  \includegraphics[width=.35\textwidth]{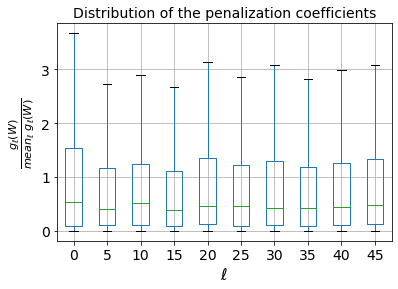}
  \caption{\small{Distribution of $g_l(\weights)$ across the layers at initialization for Vanilla ResNet50 with width 512.}}
  \vspace{-0.2cm}
  \label{fig:pen_coefficients}
\end{wrapfigure}
\cref{lemma:maximal_regularization} shows that under fixed budget, the mode $\probs^*$ that maximizes the regularization induced by \SD~is generally layer-dependent ($\neq$ uniform). However, we show that at initialization, on average (w.r.t $\weights$), $\probs^*$ is uniform.
\begin{thm}[$p^*$ is uniform at initialization]\label{thm:maximal_reg_init}
Assume $\phi=\textrm{ReLU}$ and $\weights$ are initialized with $\normalD(0, \frac{2}{N})$. Then, in the infinite width limit, under \cref{assumption:gradient_independence}, for all $l\in[1:L]$, we have 
$$\E_{\weights}[g_l(\weights)] = \E_{\weights}[g_1(\weights)].$$
As a result, given a budget $\bar{L}$, the average regularization term  $\frac{1}{2L}\sum_{l=1}^Lp_l(1-p_l)\E_{\weights}[g_l(\weights)]$ is maximal for the uniform mode $\probs^* = (\bar{L}/L, \dots, \bar{L}/L)$.
\end{thm}
The proof of \cref{thm:maximal_reg_init} is based on some results from the signal propagation theory in deep neural network. We provide an overview of this theory in Appendix \ref{app:overview_signal_prop}. \cref{thm:maximal_reg_init} shows that, given a training budget $\bar{L}$ and a randomly initialized ResNet with $\normalD(0, 2/N)$ and $N$ large, the average (w.r.t $\weights$) maximal regularization at initialization is almost achieved by the uniform mode. This is because the coefficients $\E_{\weights}[g_l(\weights)]$ are equal under \cref{assumption:gradient_independence}, which we highlight in \cref{fig:pen_coefficients}. As a result, we would intuitively expect that the uniform mode performs best when the budget $\bar{L}$ is large, e.g. $L$ is large and $\bar{L}\approx L$, since in this case, at each iteration, we update the weights of an overparameterized subnetwork, which would require more regularization compared to the small budget regime. We formalize this intuition in \cref{sec:regimes}. 


In the next section, we show that \SD~is linked to another regularization effect that only occurs in the large depth limit; in this limit, we show that \SD~\emph{mimics Gaussian Noise Injection methods by adding Gaussian noise to the pre-activations}.

\subsection{Stochastic Depth mimics Gaussian noise injection}
Recent work by \cite{camuto2020noiseinjection} studied the regularization effect of Gaussian Noise Injection (GNI) on the loss function and showed that adding isotropic Gaussian noise to the activations $z_l$ improves generalization by acting as a regularizer on the loss. The authors suggested adding a zero mean Gaussian noise parameterized by its variance. At training time $t$, this translates to replacing $z_l^t$ by $z_l^t + \normalD(0, \sigma_l^2 I)$, where $z_l^t$ is the value of the activations in the $l^{th}$ layer at training time $t$, and $\sigma_{l}^2$ is a parameter that controls the noise level. Empirically, adding this noise tends to boost the performance by making the model robust to over-fitting. Using similar perturbation analysis as in the previous section, we show that when the depth is large, \emph{\SD~mimics GNI by implicitly adding a non-isotropic data-adaptive Gaussian noise to the pre-activations $y_l$ at each training iteration}. We bring to the reader's attention that the following analysis holds throughout the training (it is not limited to the initialization), and does not require the infinite-width regime.

Consider an arbitrary neuron $y^i_{\alpha L}$ in the $(\alpha L)^{th}$ layer for some fixed $\alpha \in (0,1)$. $\neuron(x,\mask)$ can be approximated using a first order Taylor expansion around $\mask=\ones$. We obtain similarly,
\begin{equation}
\begin{split}
\neuron(x,\mask)
&\approx \bar{y}^{i}_{\alpha L}(x) +  \frac{1}{\sqrt{L}}\sum_{l=1}^{\alpha L} \eta_l \, \langle z_l, \nabla_{y_l} G_l^i(y_l(x;\ones)) \rangle 
\end{split}
\end{equation}


where $G_l^i$ is defined by $\neuron(x; \ones) = G_l^i(y_l(x;\ones))$, $\eta_l = \delta_l - p_l$, and  $\bar{y}^i_{\alpha L}(x) =\neuron(x,\ones) +  \frac{1}{\sqrt{L}}\sum_{l=1}^{\alpha L} (p_l - 1) \, \langle z_l, \nabla_{y_l} G_l^i(y_l(x;\ones)) \rangle \approx \neuron(x,\probs)$.

Let $\gamma_{\alpha,L}(x) = \frac{1}{\sqrt{L}}\sum_{l=1}^{\alpha L} \eta_l \, \langle z_l, \nabla_{y_l} G_l^i(y_l(x;\ones)) \rangle$. With \SD, $\neuron(x; \mask)$ can therefore be seen as a perturbed version of $\neuron(x;\probs)$ (the pre-activation of the average network) with noise $\gamma_{\alpha,L}(x)$. The scaling factor $1/\sqrt{L}$ ensures that $\gamma_{\alpha,L}$ remains bounded (in $\ell_2$ norm) as $L$ grows. Without this scaling, the variance of $\gamma_{\alpha,L}$ will generally explode.
The term $\gamma_{\alpha,L}$ captures the randomness of the binary mask $\mask$, which up to a factor $\alpha$, resembles to the scaled mean in Central Limit Theorem(CLT) and can be written as
$
\gamma_{\alpha,L}(x) = \sqrt{\alpha} \times \frac{1}{\sqrt{\alpha L}} \sum_{l=2}^{\alpha L} X_{l,L}(x)
$
where $X_{l,L}(x) = \eta_l \,\langle z_l, \nabla_{y_l} G_l^i(y_l(x;\ones)) \rangle$. Ideally, we would like to apply CLT to conclude on the Gaussianity of $\gamma_{\alpha,L}(x)$ in the large depth limit. However, the random variables $X_l$ are generally not $i.i.d$ (they have different variances) and they also depend on $L$. Thus, standard CLT argument fails. Fortunately, there is a more general form of CLT known as Lindeberg's CLT which we use in the proof of the next theorem.
\begin{thm}\label{thm:asymptotic_normality_of_noise}
Let $x \in \reals^d$, $X_{l,L}(x) = \eta_l \, \mu_{l,L}(x)$ where $\mu_{l,L}(x)= \langle z_l, \nabla_{y_l}  G^i_l(y_l(x;\ones)\rangle$, and $\sigma_{l, L}^2(x) = \textrm{Var}_\delta[X_{l, L}(x)] = p_l(1-p_l) \mu_{l,L}(x)^2$ for $l \in [L]$. Assume that

\hspace{0.6cm}$1.$ There exists $a \in (0,1/2)$ such that for all $L$, and $l \in [L]$, $p_l \in (a,1-a)$.

\hspace{0.6cm}$2.$ $\lim_{L\to \infty}\frac{\max_{k \in [L]} \mu_{k,L}^2(x)}{\sum_{l=1}^L \mu_{l,L}^2(x)} = 0$.

\hspace{0.6cm}$3.$ $ v_{\alpha, \infty}(x) := \lim_{L \to \infty} \frac{\sum_{l=1}^L \sigma_{l,L}^2(x)}{L}$ exists and is finite.

Then, \quad $
\gamma_{\alpha,L}(x) \underset{L \to \infty}{\overset{D}{\longrightarrow}} \normalD(0, \alpha \, v_{\alpha,\infty}(x)).
$
\end{thm}
\begin{wrapfigure}{r}{0.4\textwidth}
  \centering
  \vspace{-2.99cm}
  \includegraphics[width=.4\textwidth]{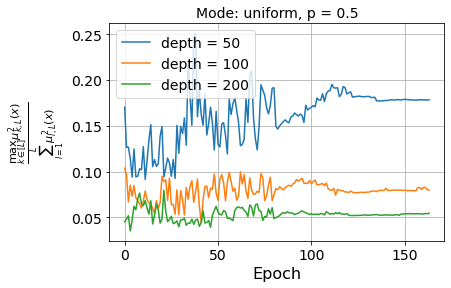}
  \caption{\small{ (\cref{thm:asymptotic_normality_of_noise}) Assumption 2 as a function the depth $L$ and epoch. }}
  \label{fig:gaussian_assumption}
  \vspace{-0.2cm}
\end{wrapfigure}
\cref{fig:gaussian_assumption} provides an empirical verification of the second condition of \cref{thm:asymptotic_normality_of_noise} across all training epochs. There is a clear downtrend as the depth increases; this trend is consistent throughout training, which supports the validity of the second condition in \cref{thm:asymptotic_normality_of_noise} at all training times. \cref{thm:asymptotic_normality_of_noise} shows that training a ResNet with \SD~involves implicitly adding the noise $\gamma_{\alpha,L}(x)$ to $\neuron$. This noise becomes asymptotically normally distributed\footnote{The limiting variance $v_{\alpha,\infty}(x)$ depends on the input $x$, suggesting that $\gamma_{\alpha,L}(.)$ might converge in distribution to a Gaussian process in the limit of large depth, under stronger assumptions. We leave this for future work.}, confirming that \SD~implicitly injects input-dependent Gaussian noise in this limit. \cite{camuto2020noiseinjection} studied GNI in the context of input-independent noise and concluded on the benefit of such methods on the overall performance of the trained network. We empirically confirm the results of \cref{thm:asymptotic_normality_of_noise} in \cref{sec:experiments} using different statistical normality tests.

Similarly, we study the \emph{implicit} regularization effect of \SD~induced on the gradient in Appendix \ref{app:implicit_regularization}, and show that under some assumptions, \SD~acts implicitly on the gradient by adding Gaussian noise in the large depth limit.



\section{The Budget Hypothesis}\label{sec:regimes}
We have seen in \cref{Sec:regularization} that given a budget $\bar{L}$, the uniform mode is linked to maximal regularization with \SD~at initialization (\cref{thm:maximal_reg_init}). Intuitively, for fixed weights $\weights$, the magnitude of standard regularization methods such as $\|.\|_1$ or $\|.\|_2$ correlates with the number of parameters; the larger the model, the bigger the penalization term. Hence, in our case, we would require the regularization term to correlate (in magnitude) with the number of parameters, or equivalently, the number of trainable layers. Assuming $L\gg 1$, and given a fixed budget $\bar{L}$, the number of trainable layers at each training iteration is close to $\bar{L}$ (\cref{Lemma:bounds_L}). Hence, the magnitude of the regularization term should depend on how large/small the budget $\bar{L}$ is, as compared to $L$. 
\vspace{-0.3cm}
\paragraph{Small budget regime ($\bar{L}/L \ll 1$).} In this regime, the effective depth $\Ldelta$ of the subnetwork is small compared to $L$. As the ratio $\bar{L}/L$ gets smaller, the sampled subnetworks become shallower, suggesting that the regularization need not be maximal in this case, and therefore $\probs$ should not be uniform in accordance with \cref{thm:maximal_reg_init}. Another way to look at this is through the bias-variance trade-off principle. Indeed, as $\bar{L}/L \to 0$, the variance of the model decreases (and the bias increases), suggesting less regularization is needed. The increase in bias inevitably causes a deterioration of the performance; we call this the Budget-performance trade-off.
To derive a more sensible choice of $\probs$ for small budget regimes, we introduce a new \emph{Information Discrepancy} based algorithm (\cref{Sec:regularization}). We call this algorithm \emph{Sensitivity Mode} or briefly \emph{SenseMode}. This algorithm works in two steps:
\begin{enumerate}
    \item Compute the sensitivity ($\mathcal{S}$) of the loss w.r.t the layer at initialization using the approximation,
    $$
\mathcal{S}_l = \loss(\weights;\ones) - \loss(\weights;\ones_{l}) \approx \nabla_{\delta_l} \loss(\weights;\mask)_{|\mask=\ones}.
$$
$\mathcal{S}_l$ is a measure of the sensitivity of the loss to keeping/removing the $l^{th}$ layer.
    \item Use a mapping $\varphi$ to map $\mathcal{S}$ to the mode,
    $
    \probs = \varphi(\mathcal{S})
    $, where $\varphi$ is a linear mapping from the range of $S$ to $[p_{min}, 1]$ and $p_{min}$ is the minimum survival rate (fixed by the user).
\end{enumerate}
\vspace{-0.3cm}
\paragraph{Large budget regime ($\bar{L}/L \sim 1$).} In this regime, the effective depth $\Ldelta$ of the subnetworks is close to $L$, and thus, we are in the overparameterized regime where maximal regularization could boost the performance of the model by avoiding over-fitting. Thus, we anticipate the uniform mode to perform better than other alternatives in this case. We are now ready to formally state our hypothesis,
\vspace{-0.3cm}
\paragraph{Budget hypothesis.} \emph{Assuming $L \gg1$, the uniform mode outperforms SenseMode in the large budget regime, while SenseMode outperforms the uniform mode in the small budget regime.}

We empirically validate the Budget hypothesis and the Budget-performance trade-off in \cref{sec:experiments}.

\section{Experiments}\label{sec:experiments}
\vspace{-0.2cm}
The objective of this section is two-fold: we empirically verify the theoretical analysis developed in sections \ref{Sec:large_width} and \ref{Sec:regularization} with a Vanilla ResNet model on a toy regression task; we also empirically validate the Budget Hypothesis on the benchmark datasets CIFAR-10 and CIFAR-100 \citep{krizhevsky2009learning}. Notebooks and code to reproduce all experiments, plots and tables presented are available in the supplementary material. We perform comparisons at constant training budgets.
\vspace{-0.2cm}
\paragraph{Implementation details:} Vanilla Stable ResNet is composed of identical residual blocks each formed of a Linear layer followed by ReLU. Stable ResNet110 follows \citep{he, huang2016stochasticdepth}; it comprises three groups of residual blocks; each block consists of a sequence Convolution-BatchNorm-ReLU-Convolution-BatchNorm. We build on an open-source implementation of standard ResNet\footnote{https://github.com/felixgwu/img\_classification\_pk\_pytorch}. We scale the blocks using a factor $1/\sqrt{L}$ as described in \cref{Sec:large_width}. 
The toy regression task consists of estimating the function 
$ f_\beta : x \mapsto \sin (\beta^T x),$
where the inputs $x$ and parameter $\beta$ are in $\mathbb{R}^{256}$, sampled from a standard Gaussian.   CIFAR-10, CIFAR-100 contain 32-by-32 color images, representing respectively 10 and 100 classes of natural scene objects. We present here our main conclusions. Further implementation details and other insightful results are in the Appendix \ref{app:further_experimental_results}.
\vspace{-0.2cm}
\paragraph{Gaussian Noise Injection:} We proceed by empirically verifying the Gaussian behavior of the neurons as described in  \cref{thm:asymptotic_normality_of_gradient_noise}. For each input $x$, we sample 200 masks and the corresponding $y(x; \boldsymbol{\delta})$. We then compute the p-value $pv_x$ of the Shapiro-Wilk test of normality \citep{shapiro1965analysis}. In \cref{fig:gaussian_output_summary} we represent the distribution of the p-values $\{pv_x\ |\ x \in \mathcal{X} \}$. We can see that the Gaussian behavior holds throughout training (left).  On the right part of the Figure, we can see that the Normal behavior becomes accurate after approximately 20 layers. In the Appendix we report further experiments with different modes, survival rates, and a different test of normality to verify both \cref{thm:asymptotic_normality_of_noise} and the critical assumption 2.

\begin{figure}	
    \vspace{-1em}
	\centering
	\begin{subfigure}[t]{.35\textwidth}
		\centering
		\includegraphics[width=\textwidth]{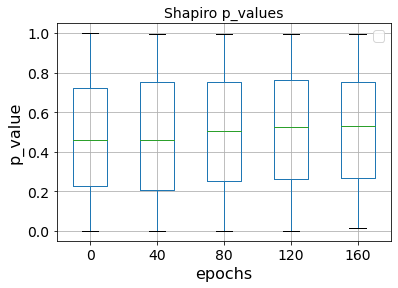}
	\end{subfigure}
	\quad
	\begin{subfigure}[t]{.35\textwidth}
		\centering
		\includegraphics[width=\textwidth]{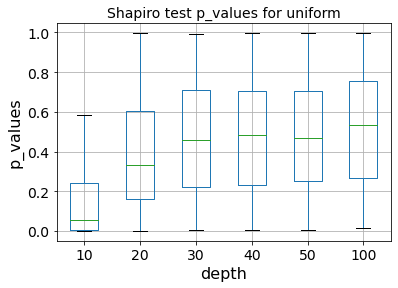}
	\end{subfigure}
	\caption{\small{Empirical verification of Theorem 2 on Vanilla ResNet100 with width 128 with average survival probability $\bar L/ L = 0.7$ and uniform mode. Distribution of the p-values for Shapiro's normality test as a function of the training epoch (left) and depth of the network (right). }}
	\label{fig:gaussian_output_summary}
	\vspace{-1em}
\end{figure}
\vspace{-0.2cm}
\paragraph{Empirical verification of the Budget Hypothesis:} We compare the three modes: Uniform, Linear, and SenseMode on two benchmark datasets using a grid survival proportions. The values and standard deviations reported are obtained using four runs. For SenseMode, we use the simple rule $p_l \propto | \mathcal{S}_l |$, where $\mathcal{S}_l$ is the sensitivity (see section \ref{sec:regimes}).
We report in Table \ref{tab:ResNet110} the results for Stable ResNet110. Results with Stable ResNet56 are reported in Appendix \ref{app:further_experimental_results}.
\begin{wraptable}{r}{5.5cm}
\vspace{-1.7em}
    \caption{\small{Comparison of the modes of selection of the survival probabilities with fixed budget with Stable ResNet110.}}
    \begin{subtable}{.4\textwidth}
  \label{growth_05}
  \centering
  \resizebox{5.5cm}{!}{%
  \begin{tabular}{cccc}
    \toprule
     $\bar L / L$ & Uniform & SenseMode & Linear \\
    \midrule
0.1 & 17.2  $\pm $ 0.3  & \textbf{15.4} $\pm $ 0.4 & $-$\\
0.2 & 10.3  $\pm $ 0.4 & \textbf{9.3} $\pm $ 0.5 & $-$\\
0.3 & 7.7  $\pm $ 0.2 & \textbf{7.0} $\pm $ 0.3 & $-$\\
0.4 & \textbf{7.4} $\pm $ 0.3 & \textbf{7.3} $\pm $ 0.4 & $-$\\
0.5 & \textbf{6.8} $\pm 0.1$ & 7.3 $\pm $ 0.2 & 9.1 $\pm $ 0.1 \\
0.6 & \textbf{6.3} $\pm 0.2$ & 6.9 $\pm $ 0.1  & 7.5 $\pm $ 0.2  \\
0.7 & \textbf{5.9} $\pm $ 0.1 & 7.3 $\pm $ 0.3  & 6.4 $\pm $  0.2  \\
0.8 & \textbf{5.7} $\pm$ 0.1 & 6.6 $\pm $ 0.2  & 6.1 $\pm $ 0.2 \\
0.9 & \textbf{5.7} $\pm$ 0.1 & 6.2 $\pm $ 0.2  & 6.0 $\pm $ 0.2 \\
\midrule
1 & \multicolumn{3}{c}{$6.37 \pm 0.12$} \\
    \bottomrule
  \end{tabular}
  }
  \caption{CIFAR10}
    \end{subtable}\\
    \begin{subtable}{.4\textwidth}
  \label{growth_05}
  \centering
  \resizebox{5.5cm}{!}{%
  \begin{tabular}{cccc}
    \toprule
     $\bar L / L$ & Uniform & SenseMode & Linear \\
    \midrule
0.1 & 55.2 $\pm$ 0.4  & \textbf{51.3} $\pm$ 0.6 & $-$\\
0.2 & 38.3 $\pm$ 0.3 & \textbf{36.4} $\pm$ 0.4 & $-$\\
0.3 & 31.4 $\pm$ 0.2 & \textbf{30.1} $\pm$ 0.5 & $-$\\
0.4 & 30.9 $\pm$ 0.2 & \textbf{28.5} $\pm$ 0.4 & $-$\\
0.5 & \textbf{28.4} $\pm$ 0.3 & 29.5 $\pm$ 0.5 & 36.5 $\pm$ 0.4 \\
0.6 & \textbf{26.5} $\pm$ 0.4 & 29.9 $\pm$ 0.6  & 30.9 $\pm$ 0.4 \\
0.7 & \textbf{25.8} $\pm$ 0.1 & 29.5 $\pm$ 0.3  & 27.3 $\pm$ 0.3  \\
0.8 & \textbf{25.5} $\pm$ 0.1 & 30.0 $\pm$ 0.3  & 25.7 $\pm$ 0.2 \\
0.9 & \textbf{25.5} $\pm$ 0.3 & 28.3 $\pm$ 0.2  & \textbf{25.5} $\pm$ 0.2 \\
\midrule
1 & \multicolumn{3}{c}{ 26.5 $\pm$ 0.2} \\
    \bottomrule
  \end{tabular}
  }
    \caption{CIFAR100}
    \end{subtable}
    \label{tab:ResNet110}
\vspace{-4em}
\end{wraptable}
The empirical results are coherent with the Budget Hypothesis. When the training budget is large, i.e. $\bar{L}/L \geq 0.5$, the Uniform mode outperforms the others. We note nevertheless that when $\bar{L}/L \geq 0.9$, the Linear and Uniform models have similar performance. This seems reasonable as the Uniform and Linear probabilities become very close for such survival proportions. When the budget is low, i.e. $\bar{L}/L < 0.5$, the SenseMode outperforms the uniform one (the linear mode cannot be used with budgets $\bar{L} <L/2$ when $L\gg1$, since $\sum p_l / L>1/2 - 1/(2L) \sim 1/2$), thus confirming the Budget hypothesis. \cref{tab:ResNet110} also shows a clear Budget-performance trade-off.

\section{Related work}
\vspace{-0.3cm}
The regularization effect of Dropout in the context of linear models has been the topic of a stream of papers \citep{wager2013dropout, mianjy2019dropout, helmbold2015inductive, cavazza2017dropout}. This analysis has been recently extended to neural networks by \cite{wei2020dropout} where authors used a similar approach to ours to depict the explicit and implicit regularization effects of Dropout. To the best of our knowledge, our paper is the first to provide analytical results for the regularization effect of \SD, and study the large depth behaviour of \SD, showing that the latter mimics Gaussian Noise Injection \citep{camuto2020noiseinjection}. A further analysis of the implicit regularization effect of \SD~is provided in Appendix \ref{app:implicit_regularization}.

\section{Limitations and extensions}\label{sec:limitations}
\vspace{-0.3cm}
In this work, we provided an analytical study of the regularization effect of \SD~using a second order Taylor approximation of the loss function. Although the remaining higher order terms are usually dominated by the second order approximation (See the quality of the second order approximation in Appendix \ref{app:further_experimental_results}), they might also be responsible for other regularization effects. This is a common limitation in the literature on noise-based regularization in deep neural networks \citep{camuto2020noiseinjection, wei2020dropout}. Further research is needed to isolate the effect of higher order terms.\\
We also believe that SenseMode opens an exciting research direction knowing that the low budget regime has not been much explored yet in the literature. We believe that one can probably get even better results with more elaborate maps $\varphi$ such that $\boldsymbol{p} = \varphi(\mathcal{S})$.
Another interesting extension of an algorithmic nature is the dynamic use of \emph{SenseMode} throughout training. We are currently investigating this topic which we leave for future work.


\newpage





\newpage

\bibliography{sample.bib}

\begin{thebibliography}{38}
\providecommand{\natexlab}[1]{#1}
\providecommand{\url}[1]{\texttt{#1}}
\expandafter\ifx\csname urlstyle\endcsname\relax
  \providecommand{\doi}[1]{doi: #1}\else
  \providecommand{\doi}{doi: \begingroup \urlstyle{rm}\Url}\fi

\bibitem[Huang et~al.(2016)Huang, Sun, Liu, Sedra, and
  Weinberger]{huang2016stochasticdepth}
Gao Huang, Yu~Sun, Zhuang Liu, Daniel Sedra, and Kilian~Q Weinberger.
\newblock Deep networks with stochastic depth.
\newblock In \emph{European conference on computer vision}, pages 646--661.
  Springer, 2016.

\bibitem[Hinton et~al.(2012)Hinton, Srivastava, Krizhevsky, Sutskever, and
  Salakhutdinov]{hinton2012improving}
Geoffrey~E Hinton, Nitish Srivastava, Alex Krizhevsky, Ilya Sutskever, and
  Ruslan~R Salakhutdinov.
\newblock Improving neural networks by preventing co-adaptation of feature
  detectors.
\newblock \emph{arXiv preprint arXiv:1207.0580}, 2012.

\bibitem[Srivastava et~al.(2014)Srivastava, Hinton, Krizhevsky, Sutskever, and
  Salakhutdinov]{srivastava2014dropout}
Nitish Srivastava, Geoffrey Hinton, Alex Krizhevsky, Ilya Sutskever, and Ruslan
  Salakhutdinov.
\newblock Dropout: a simple way to prevent neural networks from overfitting.
\newblock \emph{The journal of machine learning research}, 15\penalty0
  (1):\penalty0 1929--1958, 2014.

\bibitem[Wan et~al.(2013)Wan, Zeiler, Zhang, Le~Cun, and
  Fergus]{wan2013regularization}
Li~Wan, Matthew Zeiler, Sixin Zhang, Yann Le~Cun, and Rob Fergus.
\newblock Regularization of neural networks using dropconnect.
\newblock In \emph{International conference on machine learning}, pages
  1058--1066, 2013.

\bibitem[Webb(1994)]{webb1994noise}
A.R. Webb.
\newblock Functional approximation by feed-forward networks: a least-squares
  approach to generalization.
\newblock \emph{IEEE Transactions on Neural Networks}, 5\penalty0 (3):\penalty0
  363--371, 1994.

\bibitem[Bishop(1995)]{bishop1995noise}
Chris~M. Bishop.
\newblock Training with noise is equivalent to tikhonov regularization.
\newblock \emph{Neural Computation}, 7\penalty0 (1):\penalty0 108--116, 1995.

\bibitem[Camuto et~al.(2020)Camuto, Willetts, Simsekli, Roberts, and
  Holmes]{camuto2020noiseinjection}
A.~Camuto, M.~Willetts, U.~Simsekli, S.~J. Roberts, and C.~C. Holmes.
\newblock Explicit regularisation in gaussian noise injections.
\newblock In \emph{Advances in Neural Information Processing Systems},
  volume~33, pages 16603--16614, 2020.

\bibitem[Wager et~al.(2013)Wager, Wang, and Liang]{wager2013dropout}
Stefan Wager, Sida Wang, and Percy~S Liang.
\newblock Dropout training as adaptive regularization.
\newblock In \emph{Advances in neural information processing systems}, pages
  351--359, 2013.

\bibitem[Mianjy and Arora(2019)]{mianjy2019dropout}
Poorya Mianjy and Raman Arora.
\newblock On dropout and nuclear norm regularization.
\newblock \emph{arXiv preprint arXiv:1905.11887}, 2019.

\bibitem[Helmbold and Long(2015)]{helmbold2015inductive}
David~P Helmbold and Philip~M Long.
\newblock On the inductive bias of dropout.
\newblock \emph{The Journal of Machine Learning Research}, 16\penalty0
  (1):\penalty0 3403--3454, 2015.

\bibitem[Cavazza et~al.(2017)Cavazza, Morerio, Haeffele, Lane, Murino, and
  Vidal]{cavazza2017dropout}
Jacopo Cavazza, Pietro Morerio, Benjamin Haeffele, Connor Lane, Vittorio
  Murino, and Ren{\'e} Vidal.
\newblock Dropout as a low-rank regularizer for matrix factorization.
\newblock \emph{arXiv preprint arXiv:1710.05092}, 2017.

\bibitem[Wei et~al.(2020)Wei, Kakade, and Ma]{wei2020dropout}
C.~Wei, S.~Kakade, and T.~Ma.
\newblock The implicit and explicit regularization effects of dropout.
\newblock In \emph{Proceedings of the 37th International Conference on Machine
  Learning}, Proceedings of Machine Learning Research, pages 10181--10192.
  PMLR, 2020.

\bibitem[He et~al.(2015)He, Zhang, Ren, and Sun]{he_init}
K.~He, X.~Zhang, S.~Ren, and J.~Sun.
\newblock Delving deep into rectifiers: Surpassing human-level performance on
  imagenet classification.
\newblock In \emph{ICCV}, 2015.

\bibitem[Neal(1995)]{neal}
R.M. Neal.
\newblock \emph{Bayesian Learning for Neural Networks}, volume 118.
\newblock Springer Science \& Business Media, 1995.

\bibitem[Poole et~al.(2016)Poole, Lahiri, Raghu, Sohl-Dickstein, and
  Ganguli]{poole}
B.~Poole, S.~Lahiri, M.~Raghu, J.~Sohl-Dickstein, and S.~Ganguli.
\newblock Exponential expressivity in deep neural networks through transient
  chaos.
\newblock In \emph{Advances in Neural Information Processing Systems}, 2016.

\bibitem[Schoenholz et~al.(2017)Schoenholz, Gilmer, Ganguli, and
  Sohl-Dickstein]{samuel}
S.S. Schoenholz, J.~Gilmer, S.~Ganguli, and J.~Sohl-Dickstein.
\newblock Deep information propagation.
\newblock In \emph{International Conference on Learning Representations}, 2017.

\bibitem[Yang(2020)]{yang_tensor3_2020}
G.~Yang.
\newblock Tensor programs iii: Neural matrix laws.
\newblock \emph{arXiv preprint arXiv:2009.10685}, 2020.

\bibitem[Xiao et~al.(2018)Xiao, Bahri, Sohl-Dickstein, S.~Schoenholz, and
  Pennington]{xiaocnnmeanfield}
L.~Xiao, Y.~Bahri, J.~Sohl-Dickstein, S.~S.~Schoenholz, and P.~Pennington.
\newblock Dynamical isometry and a mean field theory of cnns: How to train
  10,000-layer vanilla convolutional neural networks.
\newblock \emph{ICML 2018}, 2018.

\bibitem[Hayou et~al.(2019)Hayou, Doucet, and Rousseau]{hayou}
S.~Hayou, A.~Doucet, and J.~Rousseau.
\newblock On the impact of the activation function on deep neural networks
  training.
\newblock In \emph{International Conference on Machine Learning}, 2019.

\bibitem[Hayou et~al.(2020)Hayou, Doucet, and Rousseau]{hayou_ntk}
S.~Hayou, A.~Doucet, and J.~Rousseau.
\newblock Mean-field behaviour of neural tangent kernel for deep neural
  networks.
\newblock \emph{arXiv preprint arXiv:1905.13654}, 2020.

\bibitem[Hayou et~al.(2021{\natexlab{a}})Hayou, Clerico, He, Deligiannidis,
  Doucet, and Rousseau]{hayou_stable_resnet}
S.~Hayou, E.~Clerico, B.~He, G.~Deligiannidis, A.~Doucet, and J.~Rousseau.
\newblock Stable resnet.
\newblock In \emph{Proceedings of The 24th International Conference on
  Artificial Intelligence and Statistics}, pages 1324--1332,
  2021{\natexlab{a}}.

\bibitem[Yang and Schoenholz(2017)]{yang2017meanfield}
G.~Yang and S.~Schoenholz.
\newblock Mean field residual networks: On the edge of chaos.
\newblock In \emph{Advances in Neural Information Processing Systems}, pages
  7103--7114, 2017.

\bibitem[Hayou et~al.(2021{\natexlab{b}})Hayou, Ton, Doucet, and
  Teh]{hayou_pruning}
S.~Hayou, J.F. Ton, A.~Doucet, and Y.W. Teh.
\newblock Robust pruning at initialization.
\newblock In \emph{International Conference on Learning Representations},
  2021{\natexlab{b}}.

\bibitem[Jacot et~al.(2018)Jacot, Gabriel, and Hongler]{jacot}
A.~Jacot, F.~Gabriel, and C.~Hongler.
\newblock Neural tangent kernel: Convergence and generalization in neural
  networks.
\newblock In \emph{Advances in Neural Information Processing Systems}, 2018.

\bibitem[Arora et~al.(2019)Arora, Du, Hu, Li, Salakhutdinov, and
  Wang]{arora2019exact}
S.~Arora, S.S. Du, W.~Hu, Z.~Li, R.~Salakhutdinov, and R.~Wang.
\newblock On exact computation with an infinitely wide neural net.
\newblock In \emph{Advances in Neural Information Processing Systems}, 2019.

\bibitem[Keskar et~al.(2016)Keskar, Mudigere, Nocedal, Smelyanskiy, and
  Tang]{keskar2016large}
Nitish~Shirish Keskar, Dheevatsa Mudigere, Jorge Nocedal, Mikhail Smelyanskiy,
  and Ping Tak~Peter Tang.
\newblock On large-batch training for deep learning: Generalization gap and
  sharp minima.
\newblock \emph{arXiv preprint arXiv:1609.04836}, 2016.

\bibitem[Jastrzebski et~al.(2018)Jastrzebski, Kenton, Ballas, Fischer, Bengio,
  and Storkey]{jastrzebski2018relation}
Stanislaw Jastrzebski, Zachary Kenton, Nicolas Ballas, Asja Fischer, Yoshua
  Bengio, and Amos Storkey.
\newblock On the relation between the sharpest directions of dnn loss and the
  sgd step length.
\newblock \emph{arXiv preprint arXiv:1807.05031}, 2018.

\bibitem[Yao et~al.(2018)Yao, Gholami, Lei, Keutzer, and
  Mahoney]{yao2018hessian}
Zhewei Yao, Amir Gholami, Qi~Lei, Kurt Keutzer, and Michael~W Mahoney.
\newblock Hessian-based analysis of large batch training and robustness to
  adversaries.
\newblock In \emph{Advances in Neural Information Processing Systems}, pages
  4949--4959, 2018.

\bibitem[LeCun et~al.(2012)LeCun, Bottou, Orr, and
  M{\"u}ller]{lecun2012efficient}
Yann~A LeCun, L{\'e}on Bottou, Genevieve~B Orr, and Klaus-Robert M{\"u}ller.
\newblock Efficient backprop.
\newblock In \emph{Neural networks: Tricks of the trade}, pages 9--48.
  Springer, 2012.

\bibitem[Sagun et~al.(2017)Sagun, Evci, Guney, Dauphin, and
  Bottou]{sagun2017empirical}
Levent Sagun, Utku Evci, V~Ugur Guney, Yann Dauphin, and Leon Bottou.
\newblock Empirical analysis of the hessian of over-parametrized neural
  networks.
\newblock \emph{arXiv preprint arXiv:1706.04454}, 2017.

\bibitem[Krizhevsky et~al.(2009)Krizhevsky, Hinton,
  et~al.]{krizhevsky2009learning}
Alex Krizhevsky, Geoffrey Hinton, et~al.
\newblock Learning multiple layers of features from tiny images.
\newblock 2009.

\bibitem[He et~al.(2016)He, Zhang, Ren, and Sun]{he}
K.~He, X.~Zhang, S.~Ren, and J.~Sun.
\newblock Deep residual learning for image recognition.
\newblock In \emph{IEEE Conference on Computer Vision and Pattern Recognition},
  2016.

\bibitem[Shapiro and Wilk(1965)]{shapiro1965analysis}
Samuel~Sanford Shapiro and Martin~B Wilk.
\newblock An analysis of variance test for normality (complete samples).
\newblock \emph{Biometrika}, 52\penalty0 (3/4):\penalty0 591--611, 1965.

\bibitem[Lee et~al.(2019)Lee, Xiao, Schoenholz, Bahri, Novak, Sohl-Dickstein,
  and Pennington]{lee_wide_nn_ntk}
J.~Lee, L.~Xiao, S.~Schoenholz, Y.~Bahri, R.~Novak, J.~Sohl-Dickstein, and
  J.~Pennington.
\newblock Wide neural networks of any depth evolve as linear models under
  gradient descent.
\newblock In \emph{Advances in Neural Information Processing Systems}. 2019.

\bibitem[Matthews et~al.(2018)Matthews, Hron, Rowland, Turner, and
  Ghahramani]{matthews}
A.G. Matthews, J.~Hron, M.~Rowland, R.E. Turner, and Z.~Ghahramani.
\newblock Gaussian process behaviour in wide deep neural networks.
\newblock In \emph{International Conference on Learning Representations}, 2018.

\bibitem[Lillicrap et~al.(2016)Lillicrap, Cownden, Tweed, and
  Akerman]{timothy_feedback}
T.~Lillicrap, D.~Cownden, D.~Tweed, and C.~Akerman.
\newblock Random synaptic feedback weights support error backpropagation for
  deep learning.
\newblock \emph{Nature Communications}, 7\penalty0 (13276), 2016.

\bibitem[Neelakantan et~al.(2015)Neelakantan, Vilnis, Le, Sutskever, Kaiser,
  Kurach, and Martens]{neelakantan2015adding}
A.~Neelakantan, L.~Vilnis, Quoc~V. Le, I.~Sutskever, L.~Kaiser, K.~Kurach, and
  J.~Martens.
\newblock Adding gradient noise improves learning for very deep networks.
\newblock \emph{arXiv prePrint 1511.06807}, 2015.

\bibitem[D'Agostino(1970)]{agostino1979test}
RALPH~B. D'Agostino.
\newblock {Transformation to normality of the null distribution of g1}.
\newblock \emph{Biometrika}, 57\penalty0 (3):\penalty0 679--681, 12 1970.
\newblock ISSN 0006-3444.
\newblock \doi{10.1093/biomet/57.3.679}.
\newblock URL \url{https://doi.org/10.1093/biomet/57.3.679}.

\end{thebibliography}

\newpage

\section*{Appendix}

\renewcommand\theequation{A\arabic{equation}}
\renewcommand{\thelemma}{A\arabic{lemma}}
\renewcommand{\theprop}{A\arabic{prop}}
\renewcommand{\thecorollary}{A\arabic{corollary}}
\renewcommand{\thesection}{A\arabic{section}}
\renewcommand{\thedefinition}{A\arabic{definition}}

\setcounter{equation}{0}
\setcounter{lemma}{0}
\setcounter{prop}{0}
\setcounter{corollary}{0}
\setcounter{definition}{0}
\setcounter{section}{-1}

\section{An overview of signal propagation in wide neural networks}\label{app:overview_signal_prop}
In section, we review some results and tools from the theory of signal of propagation in wide neural networks. This will prove valuable in the rest of the appendix.
\subsection{Neural Network Gaussian Process (NNGP)}
\paragraph{Standard ResNet without \SD.}
Consider a standard ResNet architecture with $L$ layers. The forward propagation of some input $x \in \reals^d$ is given by 
\begin{equation}
\begin{aligned}
y_0(x) &= \Psi_0(x, W_0) \\
y_l(x) &= y_{l-1}(x) +  \Psi_l(y_{l-1}(x), W_l), \quad 1\leq l \leq L,\\
y_{out}(x) &= \Psi_{out}(y_{L}(x), W_{out}),
\end{aligned}
\end{equation}
where $W_l$ are the weights in the $l^{th}$ layer, $\Psi$ is a mapping that defines the nature of the layer, and $y_l$ are the pre-activations. we consider constant width ResNet and we further denote by $N$ the width, i.e. for all $l\in[L-1]$, $y_l \in \reals^{N}$. The output function of the network is given by $s(y_{out})$ where $s$ is some convenient mapping for the learning task, e.g. the Softmax mapping for classification tasks. We denote by $o$ the dimension of the network output, i.e. $s(y_{out}) \in \reals^{o}$ which is also the dimension of $y_{out}$. 
For our theoretical analysis, we consider residual blocks composed of a Fully Connected linear layer
$$
\Psi_l(x, W) = W \phi (x).
$$
where $\phi(x)$ is the activation function. The weights are initialized with He init \cite{he_init}, e.g. for ReLU, $W^l_{ij} \sim \mathcal{N}(0,2/N)$.

The neurons $\{y_0^i(x)\}_{i\in[1:N]}$ are iid normally distributed random variables since the weights connecting them to the inputs are iid normally distributed. Using the Central Limit Theorem, as $N_{0} \rightarrow \infty$, $y^i_{1}(x)$ becomes a Gaussian variable for any input $x$ and index $i \in [1:N]$. Additionally, the variables $\{y^i_1(x)\}_{i \in [1:N]}$ are iid. Thus, the processes $y^i_{1} (.) $ can be seen as independent (across $i$) centred Gaussian processes with covariance kernel $Q_1$. This is an idealized version of the true process corresponding to letting width $N_{0}\to \infty$. Doing this recursively over $l$ leads to similar results for $y_l^i(.)$ where $l \in [1:L]$, and we write accordingly $y_l^i \stackrel{ind}{\sim} \mathcal{GP}(0, Q_{l})$. The approximation of $y_l^i(.)$ with a Gaussian process was first proposed by \citep{neal} for single layer FeedForward neural networks and was extended recently to  multiple feedforward layers by \citep{lee_wide_nn_ntk} and \citep{matthews}. More recently, excellent work by \citep{yang_tensor3_2020} introduced a unifying framework named Tensor Programs, confirming the large-width Gaussian Process behaviour for nearly all neural network architectures. 

For any input $x \in \mathbb R^d$,  we have  $\mathbb E[y^i_l(x)] = 0$, so that the covariance kernel is given by $Q_l(x,x') = \mathbb{E}[y_l^1(x)y_l^1(x')]$. It is possible to evaluate the covariance kernels layer by layer, recursively. More precisely, assume that $y_{l-1}^i$ is a Gaussian process for all $i$. Let $x,x' \in \mathbb{R}^d$. We have that
\begin{equation*}
\begin{aligned}
Q_l(x,x') &= \mathbb{E}[y_l^1(x)y_l^1(x')]\\
&= \mathbb{E}[y_{l-1}^1(x)y_{l-1}^1(x')] + \sum_{j=1}^{N_{l-1}} \mathbb{E}[(W_l^{1j})^2 \phi(y^j_{l-1}(x))\phi(y^j_{l-1}(x'))] \\
&+ \mathbb{E}\left[\sum_{j=1}^{N_{l-1}} W_l^{1j} (y_{l-1}^1(x) \phi(y_{l-1}^1(x')) + y_{l-1}^1(x') \phi(y_{l-1}^1(x)))\right].
\end{aligned}
\end{equation*}
Some terms vanish because $\mathbb{E}[W_l^{1j}] = 0$. Let $Z_j = \sqrt{\frac{N}{2} W_l^{1j}}$. The second term can be written as 
$$
\mathbb{E}\left[\frac{2}{N} \sum_j (Z_j)^2 \phi(y^j_{l-1}(x))\phi(y^j_{l-1}(x'))\right] \rightarrow 2 \,\mathbb{E}[\phi(y^1_{l-1}(x))\phi(y^1_{l-1}(x'))]\,,
$$
where we have used the Central Limit Theorem. Therefore, the kernel $Q_l$ satisfies for all $x,x'\in \mathbb R^d$
\begin{align*}
Q_l(x,x') &= Q_{l-1}(x,x') + \mathcal{F}_{l-1}(x,x')\,,
\end{align*}
where $\mathcal{F}_{l-1}(x,x') = 2 \, \mathbb{E}[\phi(y_{l-1}^1(x))\phi(y_{l-1}^1(x'))]$.

For the ReLU activation function $\phi:x\mapsto\max(0,x)$, the recurrence can be written more explicitly as in \citep{hayou}.
Let $C_l$ be the correlation kernel, defined as 
\begin{align}\label{defc}
C_l(x,x') = \tfrac{Q_l(x,x')}{\sqrt{Q_l(x,x)Q_l(x',x')}}
\end{align}
and let $f:[-1,1]\to\mathbb{R}$ be given by 
\begin{align}\label{eq:correlation_function}
f:\gamma\mapsto \tfrac{1}{\pi}(\sqrt{1-\gamma^2}+\gamma\arcsin \gamma) + \frac{1}{2} \gamma\,.
\end{align}
The recurrence relation reads
\begin{align}\label{eq:kernel_recursion_without_SD}
\begin{split}
&Q_{l} = Q_{l-1} +\tfrac{f(C_{l-1})}{C_{l-1}}Q_{l-1}\,,\\
&Q_0(x,x') =  2 \,\tfrac{x\cdot x'}{d}\,.
\end{split}
\end{align}

\paragraph{Standard ResNet with \SD.} The introduction of the binary mask $\mask$ in front of the residual blocks slightly changes the recursive expression of the kernel $Q_l$. It ii easy to see that with \SD, the $Q_l$ follows
\begin{align}\label{eq:kernel_recursion_SD}
\begin{split}
&Q_{l} = Q_{l-1} + p_l \tfrac{f(C_{l-1})}{C_{l-1}}Q_{l-1}\,,\\
&Q_0(x,x') =  2 \,\tfrac{x\cdot x'}{d}\,.
\end{split}
\end{align}
where $f$ is given by \cref{eq:correlation_function}.

We obtain similar formulas for Stable ResNet with and without \SD. 

\paragraph{Stable ResNet without \SD.}
\begin{align}\label{eq:kernel_recursion_SD}
\begin{split}
&Q_{l} = Q_{l-1} + \frac{1}{L} \tfrac{f(C_{l-1})}{C_{l-1}}Q_{l-1}\,,\\
&Q_0(x,x') =  2 \,\tfrac{x\cdot x'}{d}\,.
\end{split}
\end{align}

\paragraph{Stable ResNet with \SD.}
\begin{align}\label{eq:kernel_recursion_SD}
\begin{split}
&Q_{l} = Q_{l-1} + \frac{p_l}{L} \tfrac{f(C_{l-1})}{C_{l-1}}Q_{l-1}\,,\\
&Q_0(x,x') =  2 \,\tfrac{x\cdot x'}{d}\,.
\end{split}
\end{align}

\subsection{Diagonal elements of the kernel $Q_l$}
\begin{lemma}[Diagonal elements of the covariance]\label{lemma:diagonal_elements}
Consider a ResNet of the form \cref{eq:standard_resnet} (standard ResNet) or \cref{eq:stable_resnet}(Stable ResNet), and let $x \in \mathbb{R}^d$. We have that for all $l \in [1:L]$,

\begin{itemize}
    \item Standard ResNet without \SD: $
Q_{l}(x,x) =  2^l Q_0(x,x) \,.
$
\item Standard ResNet with \SD: $
Q_{l}(x,x) = \prod_{k=1}^{l} (1 + p_k) Q_0(x,x) \,.
$
\item Stable ResNet without \SD: $
Q_{l}(x,x) = (1 + \frac{1}{L})^l Q_0(x,x) \,.
$
\item Stable ResNet with \SD: $
Q_{l}(x,x) = \prod_{k=1}^{l} (1 + \frac{p_k}{L}) Q_0(x,x) \,.
$
\end{itemize}

\end{lemma}
\begin{proof}
Let us prove the result for Standard ResNet with \SD. The proof is similar for the other cases. Let $x \in \reals^d$.
We know that 
$$
Q_{l}(x,x) = Q_{l-1}(x,x) +  p_l f(1) Q_{l-1}(x,x)\,,
$$
where $f$ is given by \cref{eq:correlation_function}. It is straightforward that $f(1) = 1$. This yields
$$
Q_{l}(x,x) = (1 +  p_l) Q_{l-1}(x,x)\,.
$$
we conclude by telescopic product.
\end{proof}

\subsection{\cref{assumption:gradient_independence} and gradient backpropagation}\label{app:discussion_assumption}
For gradient back-propagation, an essential assumption in the literature on signal propagation analysis in deep neural networks is that of the gradient independence which is similar in nature to the practice of feedback alignment \citep{timothy_feedback}. This assumption (\cref{assumption:gradient_independence}) allows for derivation of recursive formulas for gradient back-propagation, and it has been extensively used in literature and empirically verified; see references below. 

\paragraph{Gradient Covariance back-propagation.} \cref{assumption:gradient_independence} was used to derive analytical formulas for gradient covariance back-propagation in a stream of papers; to cite a few, \citep{hayou, samuel lee_gaussian_process, poole, xiaocnnmeanfield, yang2017meanfield}. It was validated empirically through extensive simulations that it is an excellent tool for FeedForward neural networks in \cite{samuel}, for ResNets in \cite{yang2017meanfield} and for CNN in \cite{xiaocnnmeanfield}. 

\paragraph{Neural Tangent Kernel (NTK).} \cref{assumption:gradient_independence} was implicitly used by \cite{jacot} to derive the recursive formula of the infinite width Neural Tangent Kernel (See \cite{jacot}, Appendix A.1). Authors have found that this assumption yields excellent empirical match with the exact NTK. It was also used later in \citep{arora2019exact,hayou_ntk} to derive the infinite depth NTK for different architectures. 

When used for the computation of gradient covariance and Neural Tangent Kernel, \citep{yang_tensor3_2020} proved that \cref{assumption:gradient_independence} yields the exact computation of the gradient covariance and the NTK in the limit of infinite width. We state the result for the gradient covariance formally.
\begin{lemma}[Corollary of Theorem D.1. in \citep{yang_tensor3_2020}]\label{lemma:gradient_independence}
Consider a ResNet of the form \eqref{eq:standard_resnet} or \eqref{eq:stable_resnet} with weights $W$. In the limit of infinite width, we can assume that $W^T$ used in back-propagation is independent from $W$ used for forward propagation, for the calculation of Gradient Covariance.
\end{lemma}

\begin{lemma}[Gradient Second moment]\label{lemma:gradient_backprop}
Consider a ResNet of type \cref{eq:stable_resnet} without \SD. Let $(x, t) \in \data$ be a sample from the dataset, and define the second of the gradient with $\tilde{q}_l(x,t) = \E_{\weights}\left[ \frac{\partial \loss(x, \weights)}{\partial y_l}\right]$. Then, in the limit of infinite width, we have that 
$$
\tilde{q}^l(x,t) = \left(1 +\frac{1}{L}\right)\, \tilde{q}^{l+1}(x,t)\,.
$$
As a result, for all $l \in [1:L]$, we have that 
$$
\tilde{q}^l(x,t) = \left(1 +\frac{1}{L}\right)^{L-l}\, \tilde{q}^{L}(x,t)\,.
$$
\end{lemma}
\begin{proof}
It is straighforward that 
$$
\frac{\partial \loss(x, \weights)}{\partial y_l^i} = \frac{\partial \loss(x, \weights)}{\partial y_{l+1}^i} + \frac{1}{\sqrt{L}} \sum_{j} \frac{\partial \loss(x, \weights)}{\partial y_{l+1}^j} W_{l+1}^{ji} \phi'(y_l^i(x))\,.
$$
Using lemma \ref{lemma:gradient_independence} and the Central Limit Theorem, we obtain
$$
\tilde{q}^l(x,t) = \tilde{q}^{l+1}(x,t) + \frac{2}{L} \tilde{q}^{l+1}(x,t) \mathbb{E}[\phi'(y_l^1(x))^2]\,.
$$
We conclude by observing that $\mathbb{E}[\phi'(y^l_i(x))^2] = \mathbb{P}(\mathcal{N}(0,1) > 0) = \frac{1}{2}$.
\end{proof}

\section{Proofs}\label{app:proofs}
\subsection{Proof of \cref{Lemma:bounds_L}}

\begin{manuallemma}{\ref{Lemma:bounds_L}}[Concentration of  $L_{\bm{\delta}}$]
For any $\beta \in (0,1)$, we have that with probability at least $1-\beta$,
\begin{equation}
|L_{\bm{\delta}} - L_{\bm{p}}| \leq v_{\bm{p}} \,  u^{-1}\left(\frac{\log(2/\beta)}{v_{\bm{p}}}\right)
\end{equation}
where $L_{\bm{p}} = \mathbb{E}[L_{\bm{\delta}}] = \sum_{l=1}^L p_l$, $v_{\bm{p}} = \textup{Var}[L_{\bm{\delta}}] = \sum_{l=1}^L p_l(1- p_l)$, and $u(t) = (1+t) \log(1+t) - t$.

Moreover, for a given average depth $L_{\bm{p}}=\bar{L}$, the upperbound in \cref{eq:bennett} is maximal for the uniform choice of survival probabilities $\bm{p} = \left(\frac{\bar{L}}{L}, ..., \frac{\bar{L}}{L}\right)$.
\end{manuallemma}

\begin{proof}
The concentration inequality is a simple application of Bennett's inequality: Let $X_1, ..., X_n$ be a sequence of independent random variables with finite variance and zero mean. Assume that there exists $a \in \mathbb{R}^+$ such that $X_i \leq a$ almost surely for all $i$. Define $S_n = \sum_{i=1}^n X_i$ and $\sigma_n^2 = \sum_{i=1}^n \E[X_i^2]$. Then, for any $t>0$, we have that 
$$
\mathbb{P}(|S_n| > t) \leq 2 \exp\left(-\frac{\sigma_n^2}{a^2} u\left(\frac{a t}{\sigma_n^2}\right)\right).
$$

Now let us prove the second result. Fix some $\bar{L} \in (0, L)$. We start by proving that the function $\zeta(z) = z \, u^{-1}\left( \frac{\alpha}{z}\right)$ is increasing for any fixed $\alpha>0$. Observe that $u'(t) = \log(1+t)$, so that 
$$(u^{-1})'(z) = \frac{1}{u'(u^{-1}(z))} = \frac{1 + u^{-1}(z)}{z + u^{-1}(z)}.$$

This yields
$$
\zeta'(z) = \frac{z u^{-1}\left( \frac{\alpha}{z}\right)^2 - \alpha}{\alpha + x u^{-1}\left( \frac{\alpha}{z}\right)}
$$
For $z>0$, the numerator is positive if and only if $\frac{\alpha}{z} > u\left(\sqrt{\frac{\alpha}{z}} \right)$, which is always true using the inequality $\log(1+t)< t$ for all $t> 0$.\\

Now let $\alpha = \log(2/\beta)$. Without restrictions on $\Lp$, it is straightforward that $v_{\bm{p}}$ is maximized by $\bm{p}' = (1/2, ..., 1/2)$. With the restriction $\Lp = \bar{L}$, the minimizer is the orthogonal projection of $\bm{p}'$ onto the convex set $\{ \bm{p}: \Lp = \bar{L}\}$. This projection inherits the symmetry of $\bm{p}'$, which concludes the proof.
\end{proof}

\subsection{Proof of \cref{prop:exploding_gradient}}
\begin{manualprop}{\ref{prop:exploding_gradient}}[Exploding gradient rate]
Let $\loss(x,z) = \ell(y_{out}(x; \bm{\delta}), z)$ for $(x, z) \in \reals^d \times \reals^o$, where $\ell(z,z')$ is some differentiable loss function. Let $\tilde{q}_l(x,z)= \E_{W,\bm{\delta}} \frac{\lVert \nabla_{y_l} \loss \rVert^2}{ \lVert \nabla_{y_L} \loss \rVert^2}$, where the numerator and denominator are respectively the norms of the gradients with respect to the inputs of the $l^{th}$ and $L^{th}$ layers . Then, in the infinite width limit, under \cref{assumption:gradient_independence}, for all $l \in [L]$ and $(x,z) \in \reals^d \times \reals^o$, we have 
\begin{itemize}
    \item With Stochastic Depth, $ \tilde{q}_l(x,z)  = \prod_{k=l+1}^L (1+p_k)$
    \item Without Stochastic Depth (i.e. $\mask=\ones$), $ \tilde{q}_l(x,z)  = 2^{L-l} $
\end{itemize}

\end{manualprop}

\begin{proof}
It is straightforward that with Stochastic Depth
$$
\frac{\partial \loss(x, \weights)}{\partial y_l^i} = \frac{\partial \loss(x, \weights)}{\partial y_{l+1}^i} + \delta_{l+1} \sum_{j} \frac{\partial \loss(x, \weights)}{\partial y_{l+1}^j} W_{l+1}^{ji} \phi'(y_l^i(x))\,.
$$
Denote for clarity for any neuron $i$ and layer $l$, $d_{i,l} = \frac{\partial \loss(x, \weights)}{\partial y_l^i}.$ The equation can be written
\begin{equation}\label{prop_exploding:main_eq}
    d_{i,l} = d_{i, l+1} + \delta_{l+1} \sum\limits_j d_{j, l+1} W^{ji}_{l+1} \phi'(y^i_l(x))
\end{equation}
We notice that for any $k \geq l+1$ and any $i$, $d_{i, k}$ on depends on $W_{l+1}$ through the forward pass, and that the terms $W^{ji}_{l+1}$ in equation \eqref{prop_exploding:main_eq} come from the backward pass. Therefore, Lemma  \ref{lemma:gradient_independence} entails
\begin{eqnarray*}
  \mathbb{E}_{W^{\text{backward}}_{l+1}} \frac{\lVert \nabla_{y_l} \loss \rVert^2}{ \lVert \nabla_{y_L} \loss \rVert^2} &=& \frac{\sum_i d_{i,l+1}^2 + \delta_{l+1}^2 \frac{2}{N} \sum_i \phi'(y_l^i(x))^2 \sum_j d_{j, l+1}^2}{\sum_j d_{j, L}^2} \\
  &=& \frac{\lVert \nabla_{y_{l+1}} \loss \rVert^2}{ \lVert \nabla_{y_L} \loss \rVert^2} (1 + \delta_{l+1} \frac{2}{N} \sum_i \phi'(y_l^i(x))^2),
\end{eqnarray*}
where $\frac{2}{N}$ is the variance of $W^{ji}_{l+1}$ ($N$ is the width of the network). Again using Lemma \ref{lemma:gradient_independence} and taking the expectation with respect to the remaining weights and mask concludes the proof.

\end{proof}

\subsection{Proof of \cref{lemma:maximal_regularization}}
\begin{manuallemma}{\ref{lemma:maximal_regularization}}[Maximal regularization]
Consider the empirical loss $\loss$ given by \cref{eq:empirical_loss} for some fixed weights $\weights$ (e.g. $\weights$ could be the weights at any training step of SGD).Then, for a fixed training budget $\bar{L}$, the regularization is maximal for 
$$p^*_l = \min \big( 1, \max (0, \frac{1}{2} - C g_l(\weights)^{-1} ) \big),$$
where $C$ is a normalizing constant, that has the same sign as $L-2 \bar{L}$. The global maximum is obtained for $p_l = 1/2$.
\end{manuallemma}

\begin{proof}
Let $a_l = g_l(\weights)$. Noticing that $p_l(1-p_l) = 1/4 - (p_l-1/2)^2$, it comes that
\begin{eqnarray*}
  \sum\limits_l p_l (1-p_l) a_l &=& \frac{\sum a_l}{4} - \sum_l (p_l-1/2)^2 a_l \\
  &=& \frac{\sum a_l}{4} - \left\Vert p \odot \sqrt{a} - \frac{\sqrt{a}}{2} \right\Vert_2^2,
\end{eqnarray*}
with the abuse of notations $\sqrt{a}= (\sqrt{a_1}, ..., \sqrt{a_L})$ and where $\odot$ stands for the element-wise product. We deduce from this expression that $p=\bm{\frac{1}{2}}$ is the global maximizer of the regularization term. With a fixed training budget, notice that the expression is maximal for $p^* \odot \sqrt{a}$ the orthogonal projection of $\frac{\sqrt{a}}{2}$ on the intersection of the affine hyper-plane $\mathcal H$ containing the $L$ points of the form $(0,..., L_m \sqrt{a_l}, ..., 0)$ and the hyper-cube $\mathcal C$ of the points with coordinates in $[0,1]$. Writing the KKT conditions yields for every $l$:
$$ p_l = 1/2 - \beta a_l^{-1} - \lambda_{0, l} \mathbb{1}_{p_l = 0} + \lambda_{1, l} \mathbb{1}_{p_l = 1}, $$
where $\lambda_{0,l}, \lambda_{1, l} \geq 0.$ Taking $p = p^*$, $\beta = C$, $\lambda_{0, l} = 1/2 + \beta a_l^{-1}$ and $\lambda_{1, l} = 1/2 - \beta a_l^{-1}.$ Since the program is quadratic, the KKT conditions are necessary and sufficient conditions, which concludes the proof. 

\end{proof}

\subsection{Proof of \cref{thm:maximal_reg_init}}
\begin{manualthm}{\ref{thm:maximal_reg_init}}[$p^*$ is uniform at initialization]
Assume $\phi=\textrm{ReLU}$ and $\weights$ are initialized with $\normalD(0, \frac{2}{N})$. Then, in the infinite width limit, under \cref{assumption:gradient_independence}, for all $l\in[1:L]$, we have 
$$\E_{\weights}[g_l(\weights)] = \E_{\weights}[g_1(\weights)].$$
As a result, given a budget $\bar{L}$, the average regularization term  $\frac{1}{2L}\sum_{l=1}^Lp_l(1-p_l)\E_{\weights}[g_l(\weights)]$ is maximal for the uniform mode $\probs^* = (\bar{L}/L, \dots, \bar{L}/L)$.
\end{manualthm}

\begin{proof}
Using the expression of $g_l$, we have that
$$
\E_{\weights}[g_l(\weights)] = \frac{2}{n} \sum_{i=1}^n \E_{\weights}[\|\zeta_l(x_i, \weights)\|_2^2].
$$
Thus, to conclude, it is sufficient to show that for some arbitrary $i \in [1:n]$, we have $\E_{\weights}[\|\zeta_l(x_i, \weights)\|_2^2] = \E_{\weights}[\|\zeta_1(x_i, \weights)\|_2^2]$ for all $l \in [1, L]$. \\
Fix $i \in [1:n]$ and $l\in[1:L]$. For the sake of simplicity, let $z_l = z_l(x;\ones)$ and $y_l=y_l(x;\ones)$. We have that
$$
\|\zeta_l(x,\weights)\|_2^2 = \sum_{j=1}^o \langle \nabla_{y_l} G_l^j(y_l), z_l \rangle^2
$$
Now let $j \in [1:o]$. We have that 
$$
\langle \nabla_{y_l} G_l^j(y_l), z_l \rangle^2 = \sum_{k,k'} \frac{\partial G_l^j}{\partial y_l^k} z_l^k \frac{\partial G_l^j}{\partial y_l^{k'}} z_l^{k'}
$$
Using \cref{assumption:gradient_independence}, in the limit $N \rightarrow \infty$, we obtain
$$
\E_{\weights}[\langle \nabla_{y_l} G_l^j(y_l), z_l \rangle^2] = \sum_{k,k'} \E_{\weights}\left[\frac{\partial G_l^j}{\partial y_l^k}  \frac{\partial G_l^j}{\partial y_l^{k'}}\right] \E_{\weights}[z_l^k z_l^{k'}]
$$
Since $\E_{\weights}[z_l^k z_l^{k'}] = 0$ for $k \neq k'$, we have that 
$$
\E_{\weights}[\langle \nabla_{y_l} G_l^j(y_l), z_l \rangle^2] = \sum_{k=1}^\infty \E_{\weights}\left[ \left(\frac{ \partial G_l^j}{\partial y_l^k}\right)^2\right] \E_{\weights}[(z_l^k)^2]
$$
Let us deal first with the term $\E_{\weights}[(z_l^k)^2]$.

Let $k \in \mathbb{N}$. Recall that for finite $N$, we have that $z_l^k = \sum_{m=1}^N W_l^{k,m} \phi(y_{l-1}^m)$ and $W_l^{k,m} \sim \normalD(0, \frac{2}{N})$. Thus, using the Central Limit Theorem in the limit $N \rightarrow \infty$, we obtain
$$
\E_{\weights}[(z_l^k)^2] = 2 \, \E_{Z\sim \normalD(0, 1)}[\phi(\sqrt{q_{l-1}} Z)^2] = Q_{l-1}(x,x),
$$
where $Q_{l-1}(x,x)$ is given by \cref{lemma:diagonal_elements}. We obtain $$\E_{\weights}[(z_l^k)^2] = \left(1 + \frac{1}{L}\right)^{l-1} Q_0(x,x).$$

The term $\bar{q}^j_l = \E_{\weights}\left[ \left(\frac{ \partial G_l^j}{\partial y_l^k}\right)^2\right]$ can be computed in a similar fashion to that of the proof of \cref{lemma:gradient_backprop}. Indeed, using the same techniques, we obtain 
$$
\bar{q}^{j,k}_l = \left( 1 + \frac{1}{L}\right) \bar{q}^{j,k}_{l+1}
$$
which yields
$$
\bar{q}^{j,k}_l = \left( 1 + \frac{1}{L}\right)^{L-l-1} \bar{q}^{j,k}_{L}
$$

Using the Central Limit Theorem again in the last layer, we obtain 
$$
\sum_{k=1}^\infty \E_{\weights}\left[ \left(\frac{ \partial G_l^j}{\partial y_l^k}\right)^2\right]  = \left( 1 + \frac{1}{L}\right)^{L-l}.
$$
This yields
$$
\E_{\weights}[\langle \nabla_{y_l} G_l^j(y_l), z_l \rangle^2] = \left( 1 + \frac{1}{L}\right)^{L-1} Q_0(x,x).
$$
and we conclude that 
$$
\E_{\weights}[\|\zeta_l(x_i, \weights)\|_2^2] = o\times \left( 1 + \frac{1}{L}\right)^{L-1} Q_0(x,x)
$$
The latter is independent of $l$, which concludes the proof.

\end{proof}

\subsection{Proof of \cref{thm:asymptotic_normality_of_noise}}
Consider an arbitrary neuron $y^i_{\alpha L}$ in the $(\alpha L)^{th}$ layer for some fixed $\alpha \in (0,1)$. $\neuron(x,\mask)$ can be approximated using a first order Taylor expansion around $\mask=\ones$. We obtain similarly,
\begin{equation}
\begin{split}
\neuron(x,\mask)
&\approx \bar{y}^{i}_{\alpha L}(x) +  \frac{1}{\sqrt{L}}\sum_{l=1}^{\alpha L} \eta_l \, \langle z_l, \nabla_{y_l} G_l^i(y_l(x;\ones)) \rangle 
\end{split}
\end{equation}


where $G_l^i$ is defined by $\neuron(x; \ones) = G_l^i(y_l(x;\ones))$, $\eta_l = \delta_l - p_l$, and  $\bar{y}^i_{\alpha L}(x) =\neuron(x,\ones) +  \frac{1}{\sqrt{L}}\sum_{l=1}^{\alpha L} (p_l - 1) \, \langle z_l, \nabla_{y_l} G_l^i(y_l(x;\ones)) \rangle \approx \neuron(x,\probs)$.

Let $\gamma_{\alpha,L}(x) = \frac{1}{\sqrt{L}}\sum_{l=1}^{\alpha L} \eta_l \, \langle z_l, \nabla_{y_l} G_l^i(y_l(x;\ones)) \rangle$. The term $\gamma_{\alpha,L}$ captures the randomness of the binary mask $\mask$, which up to a factor $\alpha$, resembles to the scaled mean in Central Limit Theorem(CLT) and can be written as
$$
\gamma_{\alpha,L}(x) = \sqrt{\alpha} \times \frac{1}{\sqrt{\alpha L}} \sum_{l=2}^{\alpha L} X_{l,L}(x)
$$
where $X_{l,L}(x) = \eta_l \,\langle z_l, \nabla_{y_l} G_l^i(y_l(x;\ones)) \rangle$. We use a Lindeberg's CLT to prove the following result.
\begin{manualthm}{\ref{thm:asymptotic_normality_of_gradient_noise}}
Let $x \in \reals^d$, $X_{l,L}(x) = \eta_l \, \mu_{l,L}(x)$ where $\mu_{l,L}(x)= \langle z_l, \nabla_{y_l}  G^i_l(y_l(x;\ones)\rangle$, and $\sigma_{l, L}^2(x) = \textrm{Var}_\delta[X_{l, L}(x)] = p_l(1-p_l) \mu_{l,L}(x)^2$ for $l \in [L]$. Assume that

\begin{enumerate}
    \item There exists $a \in (0,1/2)$ such that for all $L$, and $l \in [L]$, $p_l \in (a,1-a)$.
    \item $\lim_{L\to \infty}\frac{\max_{k \in [L]} \mu_{k,L}^2(x)}{\sum_{l=1}^L \mu_{l,L}^2(x)} = 0$.
    \item $ v_{\alpha, \infty}(x) := \lim_{L \to \infty} \frac{\sum_{l=1}^L \sigma_{l,L}^2(x)}{L}$ exists and is finite.
\end{enumerate}
Then, 
$$
\gamma_{\alpha,L}(x) \underset{L \to \infty}{\overset{D}{\longrightarrow}} \normalD(0, \alpha \, v_{\alpha,\infty}(x)).
$$
\end{manualthm}

Before proving \cref{thm:asymptotic_normality_of_noise}, we state Lindeberg's CLT for traingular arrays.

\begin{thm}[Lindeberg's Central Limit Theorem for Triangular arrays]
Let $(X_{n,1}, \dots, X_{n,n})_{n\geq1}$ be a triangular array of independent random variables, each with finite mean $\mu_{n,i}$ and finite variance $\sigma_{n,i}^2$. Define $s_n^2 = \sum_{i=1}^n \sigma_{n,i}^2$. Assume that for all $\epsilon > 0$, we have that
$$
\lim_{n \to \infty} \frac{1}{s_n^2} \sum_{i=1}^n \E[(X_{n,i} - \mu_{n,i})^2 1_{\{ |X_{n,i} - \mu_{n,i}| > \epsilon s_n\}}] = 0,
$$
Then, we have
$$
\frac{1}{s_n} \sum_{i=1}^n (X_{n,i} - \mu_{n,i}) \underset{n \to \infty}{\overset{D}{\longrightarrow}} \normalD(0,1)
$$

\end{thm}

Given an input $x \in \reals^d$, the next lemma provides sufficient conditions for the Lindeberg's condition to hold for the triangular array of random variables $X_{l,L}(x) = \eta_l \mu_{l,L}(x)$. In this context, a scaled version of $\gamma_L(x)$ converges to a standard normal variable in the limit of infinite depth.

\begin{lemma}\label{Lemma:lindeberg_condition}
Let $x \in \reals^d$, and define $X_{l,L}(x) = \eta_l \, \mu_{l,L}(x)$ where $\mu_{l,L}$ is a deterministic mapping from $\reals^d$ to $\reals$, and let $\sigma_{l, L}^2(x) = \textrm{Var}_\delta[X_{l, L}(x)]$ for $l \in [L]$. Assume that 
\begin{enumerate}
    \item There exists $a \in (0,1/2)$ such that for all $L$, and $l \in [L]$, $p_l \in (a,1-a)$.
    \item 
    $\lim_{L\to \infty}\frac{\max_{k \in [L]} \mu_{k,L}^2(x)}{\sum_{l=1}^L \mu_{l,L}^2(x)} = 0$.
\end{enumerate}
Then for all $\epsilon>0$, we have that
$$
\lim_{L \to \infty} \frac{1}{s_L^2(x)} \sum_{l=1}^L \E[X_{l,L}(x)^2 1_{\{ |X_{l,L}(x)| > \epsilon s_L(x)\}}] = 0,
$$
where $s_L^2(x) = \sum_{l=1}^L \sigma_{l,L}^2(x)$.
\end{lemma}

The proof of \cref{Lemma:lindeberg_condition} is provide in \cref{sec:proof_lemma_lindeberg}.

\begin{corollary}\label{cor:asymptotic_normality}
Under the same assumptions of \cref{Lemma:lindeberg_condition}, we have that 
$$
\frac{1}{s_L(x)} \sum_{l=1}^L X_{l,L} \underset{L \to \infty}{\overset{D}{\longrightarrow}} \normalD(0,1)
$$
\end{corollary}

The proof of \cref{thm:asymptotic_normality_of_noise} follows from \cref{cor:asymptotic_normality}.

\subsection{Proof of \cref{Lemma:lindeberg_condition}}\label{sec:proof_lemma_lindeberg}

\begin{manuallemma}{\ref{Lemma:lindeberg_condition}}
Let $x \in \reals^d$, and define $X_{l,L}(x) = \eta_l \, \mu_{l,L}(x)$ where $\mu_{l,L}$ is a deterministic mapping from $\reals^d$ to $\reals$, and let $\sigma_{l, L}^2(x) = \textrm{Var}_\delta[X_{l, L}(x)]$ for $l \in [L]$. Assume that 
\begin{enumerate}
    \item There exists $a \in (0,1/2)$ such that for all $L$, and $l \in [L]$, $p_l \in (a,1-a)$.
    \item 
    $\lim_{L\to \infty}\frac{\max_{k \in [L]} \mu_{k,L}^2(x)}{\sum_{l=1}^L \mu_{l,L}^2(x)} = 0$.
\end{enumerate}
Then for all $\epsilon>0$, we have that
$$
\lim_{L \to \infty} \frac{1}{s_L^2(x)} \sum_{l=1}^L \E[X_{l,L}(x)^2 1_{\{ |X_{l,L}(x)| > \epsilon s_L(x)\}}] = 0,
$$
where $s_L^2(x) = \sum_{l=1}^L \sigma_{l,L}^2(x)$.
\end{manuallemma}

\begin{proof}
Fix $i \in [o]$. For $l\in [L]$, we have that $\sigma_{l,L}^2 = p_l(1-p_l) \mu_{l,L}^2$. Therefore,
$$
s_L^2 = \sum_{l=1}^L p_l(1-p_l) \mu_{l,L}^2
$$
Under conditions $1$ and $2$, it is straightforward that $s_L^2=\Theta(L)$.

To simplify our notation in the rest of the proof, we fix $x \in \reals^d$ and denote by $X_l: = X_{l,L}(x)$ and $\mu_{l} := \mu_{l,L}(x)$.
Now let us prove the result. Fix $\epsilon>0$. We have that
$$
\E[(X_l)^2 1_{\{ |X_i| > \epsilon s_L\}}] = \mu_{l}^2 \, \E\left[\eta_l^2 1_{\{ |\eta_l| > \frac{\epsilon s_L}{|\mu_l|}\}}\right]
$$
Using the fact that $|\eta_l| \leq \max(p_l, 1 - p_l)$, we have 
\begin{align*}
\E\left[\eta_l^2 1_{\{ |\eta_l| > \frac{\epsilon s_L}{|\mu_l|}\}}\right] &\leq p_l(1-p_l) 1_{\{ \max(p_l, 1 - p_l) > \frac{\epsilon s_L}{|\mu_l|}\}}\\
&\leq p_l(1-p_l) \zeta_{L}
\end{align*}
where $\zeta_{L} = 1_{\{ 1-a > \frac{\epsilon s_L}{\max_{k \in [L]}|\mu_k}\}}$.\\
Knowing that  
$$\frac{s_L}{\max_{k \in [L]}|\mu_k|} \geq a \sqrt{\frac{\sum_{l=1}^L \mu_l^2}{\max_{k \in [L]}\mu_k^2}},$$ 
where the lower bound diverges to $\infty$ as $L$ increases by assumption. We conclude that
$$
 \frac{1}{s_L^2} \sum_{l=1}^L \E[(X_l)^2 1_{\{ |X_l| > \epsilon s_L\}}] \leq \zeta_L \underset{L \to \infty}{\rightarrow} 0.
$$.
\end{proof}

\section{Full derivation of explicit regularization with \SD}\label{app:full_derivation_explicit_reg}
Consider a dataset $\data = \X \times \T$ consisting of $n$ (input, target) pairs $\{(x_i, t_i)\}_{1 \leq i \leq n}$ with $(x_i, t_i) \in \reals^d \times \reals^o$. Let $\ell:\reals^d \times \reals^o \to \reals$ be a smooth loss function, e.g. quadratic loss, crossentropy loss etc. Define the model loss for a single sample $(x,t) \in \data$ by
$$
\loss(\weights,x; \mask) =  \ell(y_{out}(x;\mask), t), \quad \loss(\weights,x) = \E_{\delta} \left[ \ell(y_{out}(x;\mask), t)\right],
$$
where $\weights=(W_l)_{0\leq l \leq L}$.
With \SD, we optimize the average empirical loss given by
$$
\loss(\weights) = \frac{1}{n} \sum_{i=1}^n \E_{\delta} \left[ \ell(y_{out}(x_i;\mask), t_i)\right]
$$

To isolate the regularization effect of \SD~on the loss function, we use a second order approximation of the loss function of the model, this will allow us to marginalize out the mask $\mask$. Let $z_l(x;\mask) = \Psi_l(W_l, y_{l-1}(x; \mask))$ be the activations. For some pair $(x,t) \in \data$, the second order Taylor approximation of $\ell(y_L(x), t)$ around $\mask = \ones = (1,\dots,1)$ is given by 
\begin{equation}
\begin{split}
\ell(y_{out}(x;\mask), z) \approx \ell(y_{out}(x;\ones), z) &+ \frac{1}{\sqrt{L}}\sum_{l=1}^L (\delta_l - 1) \langle z_l(x; \ones), \nabla_{y_l}[\ell\circ G_l](y_{l}(x;\ones)) \rangle \\
&+ \frac{1}{2L}\sum_{l=1}^L (\delta_l - 1)^2  z_l(x; \ones)^{T} \nabla^2_{y_l}[\ell\circ G_l](y_{l}(x;\ones)) z_l(x; \ones)\\
\end{split}
\end{equation}
where $G_l$ is the function defined by $y_{out}(x;\ones) = G_l(y_{l-1}(x;\ones) + \frac{1}{\sqrt{L}}z_l(x;\ones))$. 
Taking the expectation with respect to $\mask$, we obtain
\begin{equation}
\loss(\bm{W},x) \approx \bar{\loss}(\bm{W},x) + \frac{1}{2L} \sum_{l=1}^L p_l(1 - p_l) g_l(\weights, x)
\end{equation}
where $\bar{\loss}(\bm{W}, x) \approx \ell(y_{out}(x;\probs), t)$ (more precisely, $\bar{\loss}(\bm{W}, x)$ is the second order Taylor approximation of $\ell(y_{out}(x;\probs), t)$ around $\probs=1$\footnote{Note that we could obtain \cref{eq:approximation_loss} using the Taylor expansion around $\mask=\probs$. However, in this case, the hessian will depend on $\probs$, which complicates the analysis of the role of $\probs$ in the regularization term.}), and $g_l(\weights, x) =  z_l(x; \ones)^{T} \nabla^2_{y_l}[\ell\circ G_l](y_{l}(x;\ones)) z_l(x; \ones)$. 

The first term $\bar{\loss}(\weights, x)$ in \cref{eq:approximation_loss} is the loss function of the average network (i.e. replacing $\mask$ with its mean $\probs$). Thus, \cref{eq:approximation_loss} shows that training with \SD~ entails training the average network with an explicit regularization term that implicitly depends on the weights $\weights$.

\section{Implicit regularization and gradient noise}\label{app:implicit_regularization}
The results of \cref{thm:asymptotic_normality_of_noise} can be generalized to the gradient noise.
Adding noise to the gradient is a well-known technique to improve generalization. It acts as an implicit regularization on the loss function. \cite{neelakantan2015adding} suggested adding a zero mean Gaussian noise parameterized by its variance. At training time $t$, this translates to replacing $\frac{\partial \loss}{\partial w_t}$ by $\frac{\partial \loss}{\partial w_t} + \normalD(0, \sigma_t^2)$, where $w_t$ is the value of some arbitrary weight in the network at training time $t$, and $\sigma_t^2 = a (1 + t)^{-b}$ for some constants $a, b >0$. As $t$ grows, the noise converge to $0$ (in $\ell_2$ norm), letting the model stabilize in a local minimum. Empirically, adding this noise tends to boost the performance by making the model robust to over-fitting. Using similar perturbation analysis as in the previous section, we show that when the depth is large, \SD \emph{mimics this behaviour by implicitly adding a Gaussian noise to the gradient at each training step}.

Consider an arbitrary weight $w$ in the network, and let $h(x; \mask) = \frac{\partial \loss(x, \weights; \mask)}{\partial w}$ be the gradient of the model loss w.r.t $w$. $h(x;\mask)$ can be approximated using a first order Taylor expansion of the loss around $\mask=\ones$. We obtain,
\begin{equation}
\begin{split}
h(x; \mask) &= \frac{\partial \loss(x, \weights; \mask)}{\partial w}\\
&\approx \frac{\partial }{\partial w} \left( \loss(x, \weights; \ones) + \frac{1}{\sqrt{L}}\sum_{l=1}^L (\delta_l - 1) \langle z_l, \nabla_{y_l} [\ell \circ G_l](y_l(x;\ones)) \rangle \right)\\
&\approx \bar{h}(x) +  \frac{1}{\sqrt{L}}\sum_{l=1}^L \eta_l \, \frac{\partial}{\partial w} \langle z_l, \nabla_{y_l} [\ell \circ G_l](y_l(x;\ones)) \rangle 
\end{split}
\end{equation}



where $\eta_l = \delta_l - p_l$, and  $\bar{h}(x) =h(x; \ones) +  \frac{1}{\sqrt{L}}\sum_{l=1}^L (p_l - 1) \, \frac{\partial}{\partial w} \langle z_l, \nabla_{y_l} [\ell \circ G_l](y_l(x;\ones)) \rangle \approx h(x; \probs)$.

Let $\gamma_L(x) = \frac{1}{\sqrt{L}}\sum_{l=1}^L \eta_l \,  \frac{\partial}{\partial w} \langle z_l, \nabla_{y_l} [\ell \circ G_l](y_l(x;\ones)) \rangle$. With \SD, the gradient $h(x; \mask)$ can therefore be seen as a perturbation of the gradient of the average network $h(x;\probs)$ with a noise encoded by $\gamma_L(x)$. The scaling factor $1/\sqrt{L}$ ensures that $\gamma_L$ remains bounded (in $\ell_2$ norm) as $L$ grows. Without this scaling, the variance of $\gamma_L$ generally explodes.
The term $\gamma_L$ captures the randomness of the binary mask $\mask$, which resembles to the scaled mean in Central Limit Theorem and can be written as
$$
\gamma_L(x) = \frac{1}{\sqrt{L}} \sum_{l=2}^L X_{l,L}(x)
$$
where $X_{l,L}(x) = \eta_l \, \frac{\partial}{\partial w} \langle z_l, \nabla_{y_l} [\ell \circ G_l](y_l(x;\ones)) \rangle$. Ideally, we would like to apply Central Limit Theorem (CLT) to conclude on the Gaussianity of $\gamma_L(x)$ in the large depth limit. However, the random variables $X_l$ are generally not $i.i.d$ (they have different variances) and they also depend on $L$. Thus, standard CLT argument fails. Fortunately, there is a more general form of CLT known as Lindeberg's CLT which we use in the proof of the next theorem.

\begin{thm}[Asymptotic normality of gradient noise]\label{thm:asymptotic_normality_of_gradient_noise}
Let $x \in \reals^d$, and define $X_{l,L}(x) = \eta_l \, \mu_{l,L}(x)$ where $\mu_{l,L}(x)= \frac{\partial}{\partial w} \langle z_l, \nabla_{y_l} [\ell \circ G_l](y_l(x;\ones)\rangle$, and let $\sigma_{l, L}^2(x) = \textrm{Var}_\delta[X_{l, L}(x)] = p_l(1-p_l) \mu_{l,L}(x)^2$ for $l \in [L]$. Assume that 
\begin{enumerate}
    \item There exists $a \in (0,1/2)$ such that for all $L$, and $l \in [L]$, $p_l \in (a,1-a)$.
    \item 
    $\lim_{L\to \infty}\frac{\max_{k \in [L]} \mu_{k,L}^2(x)}{\sum_{l=1}^L \mu_{l,L}^2(x)} = 0$.
    \item $ v_{\infty}(x) := \lim_{L \to \infty} \frac{\sum_{l=1}^L \sigma_{l,L}^2(x)}{L}$ exists and is finite.
\end{enumerate}
Then,
$$
\gamma_L(x) \underset{L \to \infty}{\overset{D}{\longrightarrow}} \normalD(0, v_\infty(x))
$$
As a result, \SD~ implicitly mimics regularization techniques that adds Gaussian noise to the gradient.
\end{thm}

The proof of \cref{thm:asymptotic_normality_of_gradient_noise} follows from \cref{Lemma:lindeberg_condition} in a similar fashion to the proof of \cref{thm:asymptotic_normality_of_noise}.\\
Under the assumptions of \cref{thm:asymptotic_normality_of_gradient_noise}, training a ResNet with \SD~ entails adding a gradient noise $\gamma_L(x)$ that becomes asymptotically close (in distribution) to a Gaussian random variable. The limiting variance of this noise, given by $v_\infty(x)$, depends on the input $x$, which arises the question of the nature of the noise process $\gamma_L(.)$. It turns out that under some assumptions, $\gamma_L(.)$ converges to a Gaussian process in the limit of large depth. We show this in the next proposition
\newpage

\section{Further experimental results}\label{app:further_experimental_results}

\paragraph{Implementation details:} Vanilla Stable ResNet is composed of identical residual blocks each formed of a Linear and a ReLu layer. Stable ResNet110 follows \citep{he, huang2016stochasticdepth}; it comprises three groups of residual blocks; each block consists of a sequence of layers Convolution-BatchNorm-ReLU-Convolution-BatchNorm. We build on an open-source implementation of standard ResNets\footnote{https://github.com/felixgwu/img\_classification\_pk\_pytorch}. We scale the blocks using a factor $1/\sqrt{L}$ as described in \cref{Sec:large_width}. We use the adjective non-stable to qualify models where the scaling is not performed. 
The toy regression task consists of estimating the function 
$$ f_\beta : x \mapsto \sin (\beta^T x),$$
where the inputs $x$ and parameter $\beta$ are in $\mathbb{R}^{256}$, sampled from a standard Gaussian. The output is unidimensional. CIFAR-10, CIFAR-100 contain 32-by-32 color images, representing respectively 10 and 100 classes of natural scene objects. The models are learned in 164 epochs. The Stable ResNet56 and ResNet110 use an initial learning rate of 0.01, divided by 10 at epochs 80 and 120. Parameter optimization is conducted with SGD with a momentum of 0.9 and a batch size of 128. The Vanilla ResNet models have an initial learning rate of 0.05, and a batch size of 256. We use 4 GPUs V100 to conduct the experiments. In the results of \cref{Sec:large_width} and \ref{Sec:regularization}, the expectations are empirically evaluated using $500$ Monte-Carlo (MC) samples. The boxplots are also obtained using 500 MC samples.

\paragraph{The exploding gradient of Non-Stable Vanilla ResNet:} In \cref{growth_05} and \cref{growth_07}, we empirically validate \cref{prop:exploding_gradient}. We compare the empirical values of 
$$ \frac{1}{L-l}\log \tilde{q}_l(x,z)= \frac{1}{L-l}\log \E_{W,\bm{\delta}} \frac{\lVert \nabla_{y_l} \loss \rVert^2}{ \lVert \nabla_{y_L} \loss \rVert^2},$$
and compare it to the theoretical value (in parenthesis in the tables). We consider two different survival proportions. We see an excellent match between the theoretical value and the empirical one.

Proposition \ref{prop:exploding_gradient} coupled to the concavity of $\log(1+x)$ implies that at a constant budget, the uniform rate is the mode that suffers the most from gradient explosion.  Figures \ref{fig:gradgrowth_standard_0.5} and \ref{fig:gradgrowth_standard_0.7} illustrate this phenomenon. We can see that the gradient magnitude of the uniform mode can be two orders of magnitude larger than in the linear case.  However, the Stable scaling alleviates this effect; In Figure \ref{fig:gradgrowth_stable_0.7} we can see that none of the modes suffers from the gradient explosion anymore. 

\paragraph{Second order approximation of the loss:}
\begin{wraptable}{r}{0.4\textwidth}
 \vspace{-1em}
  \caption{Empirical verification of Equation (7) with Vanilla Resnet50 with width 256 and average survival probability $\bar L / L = 0.8$ with uniform mode.}
  \label{table:loss_summary}
  \centering
  \begin{tabular}{cccc}
    \toprule
     epoch & $\left| \frac{\mathcal{L} - \bar{\mathcal{L}}}{\mathcal{L}} \right|$ & $\left| \frac{\mathcal{L} - \bar{\mathcal{L}} - pen}{\mathcal{L}} \right|$ & Ratio \\
    \midrule
0 &	0.015 &	0.003 &	$\times 5.7$ \\ 
40 &	0.389 &	0.084 &	$\times 4.6$ \\
80 &	0.651 &	0.183 &	$\times 3.5$ \\
120 &	0.856 &	0.231 &	$\times 3.7$ \\
160 &	0.884 &	0.245 &	$\times 3.6$ \\
    \bottomrule
  \end{tabular}
\end{wraptable}

In Table \ref{table:loss_summary} we empirically verify the approximation accuracy of the loss (equation \eqref{eq:empirical_loss}). $\mathcal{L}$ is the loss that is minimized when learning with \SD . $\bar{\mathcal{L}}$ is the loss of the average model $y_{out}(x; \boldsymbol{p})$. The penalization term is $\frac{1}{2 L} \sum_{l=1}^L p_l(1 - p_l) g_l(\weights)$ (more details in Section \ref{Sec:regularization}). At initialization, the loss of the average model accurately represents the SD loss; the penalization term only brings a marginal correction. As the training goes, the penalization term becomes crucial; $\bar{\mathcal{L}}$ only represents $12\%$ of the loss after convergence. We can interpret this in the light of the fact that $\bar{\mathcal{L}}$ converges to zero, whereas the penalization term does not necessarily do. We note that the second-order approximation does not capture up to $25\%$ of the loss. We believe that this is partly related to the non-PSD term $\Gamma_l$ that we discarded for the analysis.

\paragraph{Further empirical verification of assumption 2 of \cref{thm:asymptotic_normality_of_noise}:} Under some assumptions, \cref{thm:asymptotic_normality_of_noise} guarantees the asymptotic normality of the noise $\gamma$. Further empirical verifications of assumption 2 are shown in \cref{fig:further_assumption2_1} and \cref{fig:further_assumption2_2}. The downtrend is consistent throughout training and modes, suggesting that assumption 2 is realistic. In \cref{fig:gaussian_output} we plot the distributions of the pvalues of two normality tests: the Shapiro–Wilk (\cite{shapiro1965analysis}) test and the the D'Agostino's $K^2$-tests (\cite{agostino1979test}).

\paragraph{Further empirical verification of the Budget Hypothesis:} We also compare the three modes (Uniform, Linear and SenseMode) in CIFAR10 and CIFAR100 for Stable Resnet56. The results are reported in \cref{tab:ResNet56}. These results confirm the observations discussed in the main text.

\begin{table}[h]
  \caption{Empirical verification of Proposition 1 with Vanilla Resnet50 with width 512 and average survival probability $\bar L / L = 0.5$. Comparison between the empirical average growth rate of the gradient magnitude against the theoretical value (between parenthesis) at initialization. }
  \label{growth_05}
  \centering
  \begin{tabular}{cccc}
    \toprule
     & Standard & Uniform & Linear \\
     $\ell$ & & & \\
    \midrule
0 &	2.003 (2) &	1.507 (1.5) &	1.433 (1.473) \\
10 & 2.002 (2) &	1.499 (1.5) &	1.349 (1.374) \\
20 & 2.001 (2) &	1.502 (1.5) &	1.248 (1.284) \\
30 & 2.002 (2) &	1.504 (1.5) &	1.207 (1.191) \\
40	& 2.002 (2) &	1.542 (1.5) &	1.079 (1.097)\\
    \bottomrule
  \end{tabular}
\end{table}

\begin{table}[h]
  \caption{Empirical verification of Proposition 1 with Vanilla Resnet50 with width 512 and average survival probability $\bar L / L = 0.7$. Comparison between the empirical average growth rate of the gradient magnitude against the theoretical value (between parenthesis) at initialization. }
  \label{growth_07}
  \centering
  \begin{tabular}{cccc}
    \toprule
     & Standard & Uniform & Linear \\
     $\ell$ & & & \\
    \midrule
0 &	2.001 (2) &	1.705 (1.7)&	1.694 (1.691) \\
10 & 2.001 (2)&	1.708 (1.7)&	1.633 (1.629)\\
20 & 2.001 (2)&	1.707 (1.7)&	1.569 (1.573) \\
30 & 2.001 (2)&	1.716 (1.7)&	1.555 (1.516) \\
40 & 1.999 (2)&	1.739 (1.7)&	1.530 (1.459)\\
    \bottomrule
  \end{tabular}
\end{table}

\begin{figure}[h]	
	\centering
	\begin{subfigure}[t]{.48\textwidth}
		\centering
		\includegraphics[width=\textwidth]{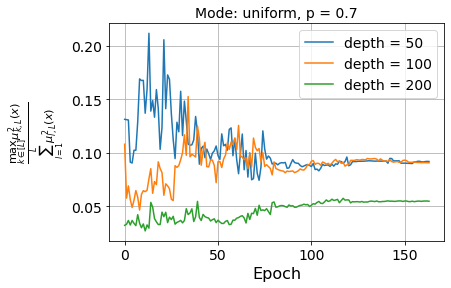}
		\caption{Uniform mode}	
	\end{subfigure}
	\quad
	\begin{subfigure}[t]{.48\textwidth}
		\centering
		\includegraphics[width=\textwidth]{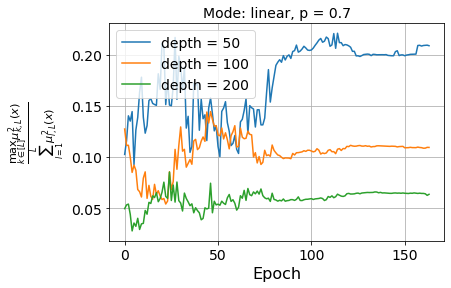}
		\caption{Linear mode}
	\end{subfigure}
	\caption{Empirical verification of assumption 2 of Theorem 2 on Vanilla ResNet with width 256 with average survival probability $\bar L/ L = 0.7$.}
	\label{fig:further_assumption2_1}
\end{figure}

\begin{figure}[h]	
	\centering
	\begin{subfigure}[t]{.48\textwidth}
		\centering
		\includegraphics[width=\textwidth]{images/lemma_4_unif_0.5.png}
		\caption{Uniform mode}	
	\end{subfigure}
	\quad
	\begin{subfigure}[t]{.48\textwidth}
		\centering
		\includegraphics[width=\textwidth]{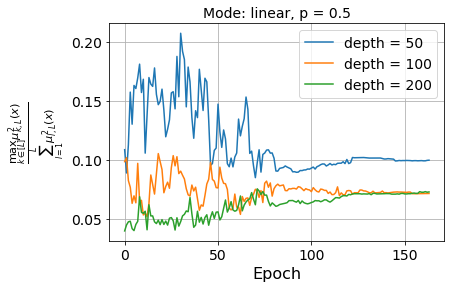}
		\caption{Linear mode}
	\end{subfigure}
	\caption{Empirical verification of assumption 2 of Theorem 2 on Vanilla ResNet with width 256 with average survival probability $\bar L/ L = 0.5$.}
	\label{fig:further_assumption2_2}
\end{figure}

\begin{figure}[h]	
	\centering
	\begin{subfigure}[t]{.48\textwidth}
		\centering
		\includegraphics[width=\textwidth]{images/shapiro_d100_p07_uniform.png}
	\end{subfigure}
	\quad
	\begin{subfigure}[t]{.48\textwidth}
		\centering
		\includegraphics[width=\textwidth]{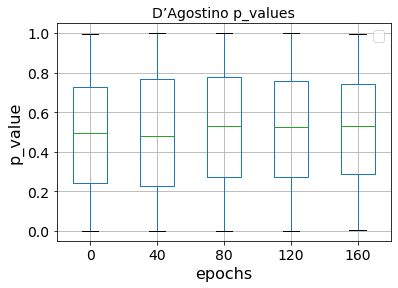}
	\end{subfigure}
	\caption{Empirical verification of Theorem 2 on Vanilla ResNet100 with width 128 with average survival probability $\bar L/ L = 0.7$ and uniform mode. Distribution of the p-values for two normality tests: Shapiro and D'Agostino's tests}
	\label{fig:gaussian_output}
\end{figure}

\begin{table}
\caption{\small{Comparison of the modes of selection of the survival probabilities with fixed budget with Stable ResNet56.}}
    \begin{subtable}{.4\textwidth}
  \label{growth_05}
  \centering
  \begin{tabular}{cccc}
    \toprule
     $\bar L / L$ & Uniform & SenseMode & Linear \\
    \midrule
0.1 & 24.58  $\pm $ 0.3  & \textbf{2.93} $\pm $ 0.4 & $-$\\
0.2 & 13.85  $\pm $ 0.3 & \textbf{11.72} $\pm $ 0.3 & $-$\\
0.3 & 10.23  $\pm $ 0.2 & \textbf{8.59} $\pm $ 0.4 & $-$\\
0.4 & \textbf{8.49} $\pm $ 0.2 & \textbf{8.23} $\pm $ 0.3 & $-$\\
0.5 & \textbf{8.38} $\pm 0.2$ & \textbf{8.25} $\pm $ 0.3 & 12.01 $\pm $ 0.3 \\
0.6 & \textbf{7.34} $\pm 0.3$ & 8.17 $\pm $ 0.2  & 9.26 $\pm $ 0.2  \\
0.7 & \textbf{8.03} $\pm $ 0.1 & 8.20 $\pm $ 0.1  & 8.30 $\pm $  0.1  \\
0.8 & \textbf{6.48} $\pm$ 0.1 & 7.55 $\pm $ 0.1  & 6.89 $\pm $ 0.2 \\
0.9 & 7.16 $\pm$ 0.1 & 7.81 $\pm $ 0.1  & \textbf{6.62} $\pm $ 0.1 \\
\midrule
1 & \multicolumn{3}{c}{$7.10 \pm 0.1$} \\
    \bottomrule
  \end{tabular}
  \caption{CIFAR10 with ResNet56}
    \end{subtable}%
    \hfill
    \begin{subtable}{.4\textwidth}
  \label{growth_05}
  \centering
  \begin{tabular}{cccc}
    \toprule
     $\bar L / L$ & Uniform & SenseMode & Linear \\
    \midrule
0.1 & 61.98	$\pm $ 0.3 & 60,27 $\pm $ 0.2 & $-$\\
0.2 & 47.24	$\pm $ 0.2 & 45.74 $\pm $ 0.3 & $-$\\
0.3 & 39.38	$\pm $ 0.2 & 37,11 $\pm $ 0.2 & $-$\\
0.4 & 35,54	$\pm $ 0.2 & 33,71 $\pm $ 0.4 & $-$\\
0.5 & 32.32	$\pm $ 0.1 & 31 $\pm $ 0.3 & 40,71 $\pm$ 0.2 \\
0.6 & 29.57	$\pm $ 0.1 & 30.19 $\pm $ 0.3	& 34.13 $\pm $ 0.1 \\
0.7 & 28.49	$\pm $ 0.4 & 29.69 $\pm $ 0.1 & 30.14 $\pm $ 0.1  \\
0.8 & 27.23	$\pm $ 0.2 & 29.31 $\pm $ 0.2 & 28.34 $\pm $ 0.2  \\
0.9 & 27.01	$\pm $ 0.1 & 29,45 $\pm $ 0.2 & 27,35 $\pm $ 0.2  \\
\midrule
1 & \multicolumn{3}{c}{ 28.93 $\pm$ 0.5} \\
    \bottomrule
  \end{tabular}
    \caption{Cifar100 with ResNet56}
    \end{subtable}
    \label{tab:ResNet56}
\end{table}

\end{document}